\documentclass[runningheads]{llncs}

\usepackage{graphicx}
\usepackage{wrapfig}
\usepackage{amsmath,amssymb} 
\usepackage{soul}
\usepackage{color}
\usepackage[space]{cite}
\usepackage{booktabs}
\usepackage{array}
\usepackage[misc]{ifsym}
\usepackage{comment}
\usepackage{multirow}
 
\usepackage{eccv}

\usepackage{eccvabbrv}

\usepackage{graphicx}
\usepackage{booktabs}

\usepackage[accsupp]{axessibility}  

\usepackage{hyperref}

\usepackage{orcidlink}

\usepackage[absolute,overlay]{textpos}

\begin{document}

	\begin{textblock*}{20cm}(2cm,2cm)
	\textcolor{red}{\textbf{Author's accepted manuscript, to appear in the proc. of ECCV24.}}
	\end{textblock*}

\newcommand{\txt}[1]{{\texttt{{#1}}}}

\newcommand{\yun}[1]{\textcolor{red}{(YUN: #1)}}

\newcommand{\xz}[1]{\textcolor{red}{(XZ: #1)}}
\newcommand{\xiaowei}[1]{{\color{blue}#1}}

\title{ProTIP: Probabilistic Robustness Verification on Text-to-Image Diffusion Models against Stochastic Perturbation} 

\titlerunning{Probabilistic Robustness Verification on Text-to-Image Diffusion Models}


\author{Yi~Zhang\inst{1}
\and
Yun~Tang\inst{1}
\and
Wenjie~Ruan\inst{2}
\and
Xiaowei Huang\inst{2}
\and
Siddartha Khastgir\inst{1}
\and
Paul Jennings\inst{1}
\and
Xingyu Zhao\inst{1}\textsuperscript{(\Letter)}
}

\authorrunning{Y.~Zhang et al.}

\institute{WMG, University of Warwick, Coventry CV4 7AL, U.K.
\email{\{yi.zhang.16,yun.tang,s.khastgir.1,paul.jennings,xingyu.zhao\}@warwick.ac.uk}\\
\and
Computer Science Department, University of Liverpool, Liverpool L69 3BX, U.K.\\
\email{w.ruan@trustai.uk}; ~\email{xiaowei.huang@liverpool.ac.uk}}

\maketitle

\begin{abstract}
Text-to-Image (T2I) Diffusion Models (DMs) excel at creating high-quality images from text descriptions but, like many deep learning models, suffer from robustness issues.
While there are attempts to evaluate the robustness of T2I DMs as a \textit{binary} or \textit{worst-case} problem, they cannot answer how robust \textit{in general} the model is whenever an adversarial example (AE) can be found.
In this study, we first formalise a \textit{probabilistic} notion of T2I DMs' robustness; and then devise an \textit{efficient} framework, ProTIP, to evaluate it with \textit{statistical guarantees}. The main challenges stem from: \textit{i)} the high computational cost of the image generation process; and \textit{ii)} identifying if a perturbed input is an AE involves comparing two output \textit{distributions}, which is fundamentally harder compared to other DL tasks like classification where an AE is identified upon misprediction of labels.
To tackle the challenges, we employ \textit{sequential analysis with efficacy and futility early stopping rules} in the statistical testing for identifying AEs, and \textit{adaptive concentration inequalities} to dynamically determine the ``just-right'' number of stochastic perturbations whenever the verification target is met.
Empirical experiments validate ProTIP's effectiveness and efficiency, and showcase its application in ranking common defence methods.

\end{abstract}

\section{Introduction}

Recent advancements in Text-to-Image (T2I) Diffusion Models (DMs), including state-of-the-arts (SOTA) like DALL-E 3 \cite{BetkerImprovingIG}, Imagen \cite{saharia2022photorealistic}, Parti \cite{yu2022scaling} and Stable Diffusion \cite{rombach2022high}, enable the generation of high-quality images from text prompts.
However, studies \cite{zhuang2023pilot,fort2021pixels,gao2023evaluating} have shown that small perturbations in textual input can substantially degrade the performance of T2I DMs. Fig.~\ref{fig_perturbation_example} illustrates examples in which introducing a minor perturbation to the prompt leads to a non-trivial change in the generated images, raising concerns about the model’s robustness in the context of its downstream applications \cite{carlini2020evading,chen2023pathway}. As such, a pivotal question arises: \textit{how can we systematically evaluate and verify (when a specific verification target is provided) the robustness of T2I DMs?}
\begin{figure}[th!]
    \centering
\includegraphics[width=0.95\textwidth]{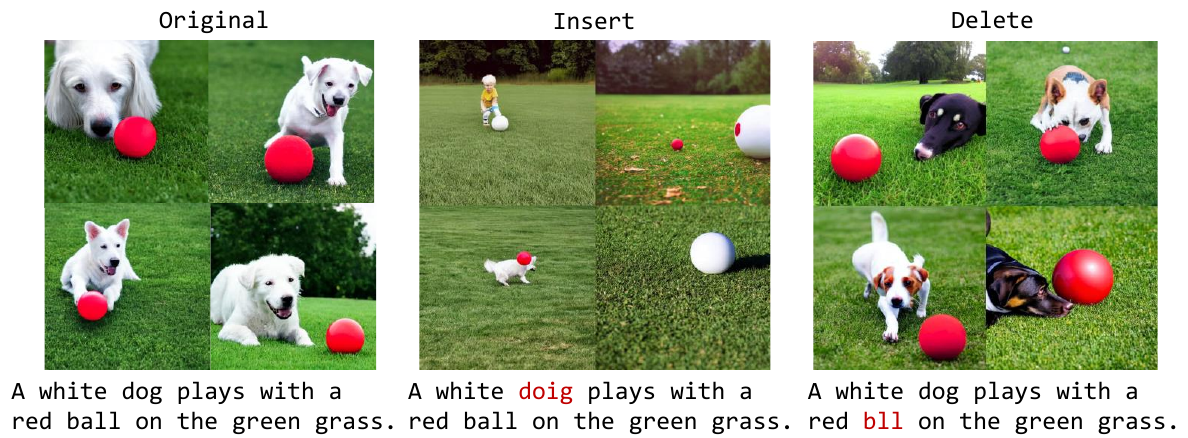}
    \caption{Examples illustrating perturbations applied to the prompt for Stable Diffusion.
    }
    \label{fig_perturbation_example}
\end{figure}

The problem of lacking robustness in T2I DMs is not unexpected, considering that robustness issues are pervasive in the broader realm of Deep Learning (DL) \cite{szegedy2013intriguing,goodfellow2014explaining}. Generally, robustness is defined as the invariant decision of the DL model against small perturbations on inputs. 
A small perturbation on an input is termed an Adversarial Example (AE) if it leads to a different prediction from the label assigned to the original input. Over the past decade, numerous studies have attempted to frame DL robustness evaluation as a \textit{binary} or \textit{worst-case} problem, addressing questions like ``if AEs exist (in a small local region of the original input)?'' or ``what is the closest AE to the original input?'' \cite{yuan2019adversarial,chakraborty2021survey,huang_survey_2020}. Recently, emerging studies are adopting a \textit{probabilistic view}, formulating robustness verification as a statistical inference problem \cite{webb2018statistical,wang2021statistically,zhang2022proa,pmlr_v206_tit23a,Huang_2023_ICCV,dong_reliability_2023,pmlr_v97_cohen19c}. Such a probabilistic robustness notion is arguably of more practical interest than binary/worst-case ones, because it provides an \textit{overall} evaluation of how robust the model is whenever an AE can be found \cite{webb2018statistical,Huang_2023_ICCV} and accepts residual risks that are more realistic to achieve \cite{wang2021statistically,zhang2022proa}. We concur with this view and argue that T2I DMs also necessitate probabilistic robustness verification, which, to the best of our knowledge, is absent in SOTA.



In this work, we first formally define the probabilistic robustness of T2I DMs against 
stochastic perturbation; then, we devise an efficient framework, named ProTIP, to evaluate and verify it (assuming a verification target is given). In addressing the statistical inference problems within ProTIP, we encounter \textit{new challenges specific to the generative nature of DMs}. These challenges, distinct from traditional DL tasks like classification or regression, include: \textit{i)} the generation process is very computationally insensitive; \textit{ii)} the identification of a perturbed input as an AE entails \textit{comparing distributional differences}, which is fundamentally harder compared to DL tasks like classification where an AE is identified by a change of the prediction label. To tackle the challenges, we employ sequential analysis with efficacy and futility early stopping rules in the statistical testing for indicating AEs, and adaptive concentration inequalities to dynamically determine the ``just-right'' number of stochastic perturbations whenever the verification target is met. Experimental results validate the effectiveness and efficiency of our ProTIP, while an application of ProTIP is also presented to showcase its use in ranking common defences against text perturbations. 

In summary, key contributions of this paper include:
\begin{itemize}
    \item \textbf{Problem formulation:} 
    For the first time, we formulate the probabilistic robustness verification problem for T2I DMs against stochastic perturbation.
    \item \textbf{Efficient solution:} 
     To solve the formulated problem,
     we develop an efficient framework, ProTIP, which incorporates several sequential analysis methods to dynamically determine the sample size and thus enhance the efficiency.
    \item  \textbf{Open-source repository:} A public repository containing the codes for ProTIP, along with datasets, models, and experimental results, is available at  \url{https://github.com/wellzline/ProTIP/}.
\end{itemize}

\section{Preliminaries and Related Work}

\subsection{Text-to-Image Diffusion Models}
Generative AI has been thriving in the multi-modal field, with DMs emerging as a powerful new family of deep generative models with record-breaking performance in many applications, including video generation \cite{wu2023tune}, image reconstruction \cite{takagi2023high}, and T2I generation \cite{BetkerImprovingIG}. DMs are a class of probabilistic generative models that apply a noise injection process, followed by a reverse procedure for sample generation. Importantly, DMs allow a notion of control at generation time through so-called guidance, enabling several complicated generative tasks, among which T2I DMs are examples of guided image generation tasks using descriptive text as the guidance signal. Models like Stability AI's Stable Diffusion \cite{rombach2022high} and Google's Imagen \cite{saharia2022photorealistic}, trained on large-scale datasets of annotated text-image pairs, can generate high-quality images based on simple natural language descriptions. 
Commercial products like DALL-E 3 \cite{BetkerImprovingIG} and Midjourney \cite{midjourney} have demonstrated remarkable capabilities in diverse text-to-image applications, enriching the field. We refer readers to Appendix A
for more details on T2I DMs.

\subsection{Deep Learning Robustness}

DL models are notoriously unrobust to small perturbations \cite{yuan2019adversarial,chakraborty2021survey,huang_survey_2020}. 
While definitions of robustness vary in literature, they share a common intuition that a DL model's decision should remain invariant against small perturbations on a given input---typically it is defined as all inputs in a region $\eta$ have the same prediction label, where $\eta$ is a small norm ball (in a $L_p$-norm distance) of radius $\gamma$ around an input $x$. A perturbed input (e.g., by adding noise on $x$) $x'$ within $\eta$ is an AE if its prediction label differs from $x$.

In the Safe AI community, DL robustness has been \textit{the} property in the spotlight. Many studies on evaluating DL robustness have been done, framing the problem in different ways. In Fig.~\ref{fig_types_robustness}, we summarise 4 common ways of formulating the problem, inspired by \cite{dong_reliability_2023}. Earlier works, such as \cite{gehr2018ai2,katz2017reluplex,ruan2018reachability}, formulate the verification problem as a binary question by asking if any AEs can be found within a given input norm-ball of a specified radius, cf. Fig.~\ref{fig_types_robustness}(a). Such ``binary robustness'' can be normally evaluated in two ways \cite{katz2017reluplex,ruan2018reachability,8318388}: either through reachability algorithms that aim to determine the lower and upper bounds of the output within the input norm-ball $\eta$, using layer-by-layer analysis; or by solving it using SAT/SMT solvers as a variety of constraint-based programming problems. Fig.~\ref{fig_types_robustness}(b) poses a similar yet different question: what is the maximum radius of $\eta$ such that no AE exists within it? Intuitively, it is finding the ``largest safe perturbation distance'' for input $x$ \cite{aminifar2020universal,moosavi2017universal,weng2019proven,weng2018evaluating}. 
While in Fig.~\ref{fig_types_robustness}(c), it evaluates the model's robustness by introducing adversarial attacks to cause the maximum prediction loss in the specified norm ball $\eta$. It is often applied in \textit{adversarial training} to enhance the model's robustness to resist attacks \cite{madry2017towards,wang2019convergence}.

\begin{figure}[ht!]
    \centering
\includegraphics[width=0.9\textwidth]{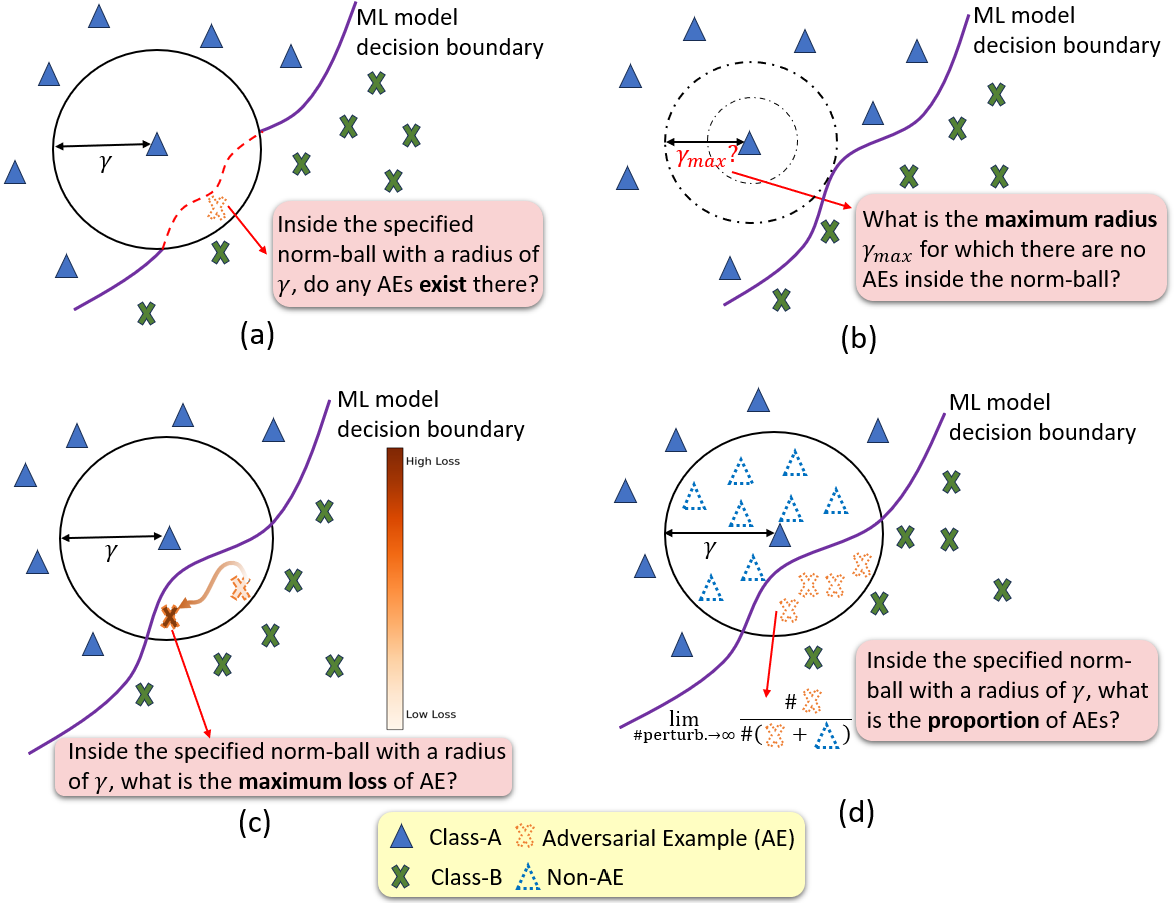}
    \caption{Four common formulations of robustness verification in DL---binary (a), worst-case (b \& c), and probabilistic (d) robustness.}
    \label{fig_types_robustness}
\end{figure}


The three aforementioned methods all estimate the robustness of the model by detecting the presence of AE or \textit{the} AE that gives the maximum loss/safe-radius, the so called \textit{deterministic} robustness \cite{zhang2022proa}. 
As argued by \cite{webb2018statistical}, they suffer from two major drawbacks: \textit{i)} they fail to convey \textit{how} robust the model is whenever an AE is found; \textit{ii)} they pose scalability challenges when the model is large. Thus, recent works \cite{webb2018statistical,wang2021statistically,zhang2022proa,pmlr_v206_tit23a,Huang_2023_ICCV,dong_reliability_2023,pmlr_v97_cohen19c} develop a new \textit{probabilistic} view, by defining robustness as the \textit{proportion} of AEs inside the norm-ball $\eta$, cf. Fig.~\ref{fig_types_robustness}(d). This probabilistic notion is arguably more practical, because: \textit{i)} binary/worst-case robustness focusing on extreme cases is neither necessary nor realistic, especially when the model is large; knowing the proportion of AEs is more relevant;
\textit{ii)} since all practical applications have acceptable levels of risk, it suffices to demonstrate that the violation probability is below a required threshold, rather than confirming it to be exactly zero. Without loss of generality (WLOG), we illustrate probabilistic robustness using a DL classification task as:
\begin{definition}[Probabilistic Robustness]
\label{def_Prob_rob_classification}
For a DL classifier $f$ that takes input $x\in \mathcal{X}$ and returns a prediction label, the probabilistic robustness of an input $x$ in a norm ball of radius $\gamma$, denoted as $B(x,\gamma)$, is:
\begin{equation}\label{eq_probabilistic_robustness_definition}
R (x, \gamma) := \int_{x' \in B(x,\gamma)} I_{\{f(x')=f(x)\}}(x') Pr (x') \, dx' 
\end{equation}
where $I_{\mathcal{S}}(x)$ is an indicator function---it is equal to 1 when $\mathcal{S}$ is true and equal
to 0 otherwise; $Pr(\cdot)$ is the local distribution of inputs representing how perturbations $x'$ are generated, which is precisely the ``input model'' used by \cite{webb2018statistical, weng2019proven}.
\end{definition}
In this paper, we re-frame such a generic definition of probabilistic robustness for T2I DMs and provide an efficient solution to its verification.

\subsection{Robustness of Text-to-Image Diffusion Models}
Despite variations of definitions, the key to evaluating and improving the model's robustness is detecting AEs. The approaches to this end are often referred to as ``adversarial attacks'' via input perturbations. Although such perturbations are commonly referred to as ``attacks'', they are \textit{not necessarily malicious actions of attackers} \cite{zhao_detecting_2021}. They may also represent natural sensor white noise \cite{hendrycks2021natural} or benign human errors that follow a \textit{stochastic} generation process. Note that in this work, we adopt such a generalised terminology of AEs to represent \textit{small and random} perturbations. For T2I models, such perturbations can be classified into character-level, word-level, sentence-level or multi-level, depending on the granularity of input perturbations. Recent studies \cite{maus2023black, fort2021pixels} show that T2I DMs are very sensitive to black-box attacks like text perturbation. In \cite{gao2023evaluating,zhuang2023pilot}, T2I DMs are shown to be vulnerable to realistic human errors (e.g., typos, glyphs, phonetic errors), exposing significant robustness issues due to weak text encoders. 

These studies focused only on \textit{deterministic} robustness (e.g., maximised prediction loss).
To the best of our knowledge, there is no dedicated exploration into the \textit{probabilistic robustness} of T2I models when they are subject to \textit{stochastic perturbations}, and our ProTIP is the first verification framework for this problem, backed by statistical guarantees. Moreover, ProTIP addresses unique challenges arising from the generative characteristics of T2I DMs, by adopting sequential analysis and adaptive concentration inequality to improve the verification efficiency, details of which are provided in the next section.

\section{Method: ProTIP}
\subsection{Problem Statement}\label{sec:problem_statement}

A T2I DM that takes a text input $x\in \mathcal{X} $ and generates an image $y \in \mathcal{Y}$ essentially characterises the conditional distribution $Pr(Y \mid X=x)$\footnote{As usual, we use capital letters to denote random variables and lower case letters for their specific realisations; $Pr(X)$ is used to represent the distribution of variable $X$.}, i.e., the T2I DM is a function $f:\mathcal{X} \rightarrow \mathcal{D}(\mathcal{Y})$ where $\mathcal{D}$ represents the space of all possible distributions over the image set $\mathcal{Y}$. Accordingly, the general probabilistic robustness Def.~\ref{def_Prob_rob_classification} needs to be adapted for T2I DMs as:
\begin{definition}[Probabilistic Robustness of T2I DMs]
\label{def_Prob_rob_t2i}
For a T2I DM $f$ that takes text inputs $X$ and generates a conditional distribution of images $Pr(Y \mid X)$, the probabilistic robustness of the given input $x$ is: 
\begin{equation}\label{eq_T2I_probabilistic_robustness_definition}
R_{M}{(x, \gamma)} = \sum_{x': s(T(x),T(x')) \geq \gamma} I_{\{Pr(Y \mid X=x ) = Pr(Y \mid X=x')\}}(x') Pr(x')
\end{equation}
where $T$ denotes the CLIP \cite{radford2021learning} model's text encoder, $s$ is a similarity measurement function (e.g., cosine similarity), $\gamma$ denotes a given threshold on similarity. While $I$ is an indication function as defined in Def.~\ref{def_Prob_rob_classification}, its value now depends on whether the output distributions before and after the perturbation differ. $Pr(x')$ indicates the probability that $x'$ is the next perturbed text generated randomly.
\end{definition}

Intuitively, Def.~\ref{def_Prob_rob_t2i} suggests that $R_M(x,\gamma)$ is the expected probability that the output image distribution remains unchanged for a random perturbation $x'$ that preserves a similar semantic meaning to $x$ (ensured by $s(T(x),T(x')) \geq \gamma$). A ``frequentist'' interpretation of $R_M(x)$\footnote{Notation-wise, we omit $\gamma$ from $R_M$ where $\gamma$ represents a hyperparameter thereafter.}, following the gist of Fig.~\ref{fig_types_robustness}(d), is: it is the \textit{limiting relative frequency} of perturbations for which the output distribution is preserved, in an infinite sequence of independently generated perturbations.

\begin{definition}[Verification Target]
\label{def_verfication}
     The probabilistic robustness of the T2I DM $f$ when processing an input $x$ cannot be less than a certain lower bound $b_l$, with sufficient confidence $1-\sigma$, i.e.:
\begin{equation}
    \label{eq_verf_target}
    Pr(R_M(x) \geq b_l) \geq 1-\sigma
\end{equation}
where the pair $(b_l,\sigma)$ is the given verification target.
\end{definition}
To determine if a verification target is met for the input $x$, we propose the ProTIP workflow in Fig.~\ref{fig_framework}, addressing the following questions: \textit{i)} How to generate stochastic perturbations $x'$ (i.e., the implementation of $Pr(x')$ in Eq.~\eqref{eq_T2I_probabilistic_robustness_definition}); \textit{ii)} Given a perturbed input $x'$, how to determine if its new output distribution is significantly different to the original one (i.e., the implementation of $I(x')$ in Eq.~\eqref{eq_T2I_probabilistic_robustness_definition}); \textit{iii)} How to do decision makings based on statistical evaluations of $R_M(x)$ over a sequence of perturbations (i.e., even if we implemented both $Pr(x')$ and $I(x')$ in Eq.~\eqref{eq_T2I_probabilistic_robustness_definition}, the true $R_M(x)$ is still unknown due to the fact that we cannot exhaustively take the sum of their product over all possible perturbations $x'$; thus, we can only estimate it from a finite sample of perturbed inputs).


While the aforementioned question \textit{i)} is established for which we adopt SOTA 
methods, e.g., \cite{morris2020textattack}, to generate stochastic perturbations on text (cf. Sec.~\ref{sec_pert_method}), questions \textit{ii)} and \textit{iii)} are relatively challenging. Because, both $Pr(Y\mid X=x)$ and $Pr(Y\mid X=x')$ are unknown non-parametric distributions that require sampling images from them (which is costly for T2I DMs) to determine if they are significantly different. To tackle, we propose the \textit{sequential analysis with early stopping rules} and \textit{adaptive concentration inequality} for the last two questions, corresponding to the ``inner loop'' and ``outer loop'' of Fig.~\ref{fig_framework}, respectively.
\begin{figure}[h!]
    \centering
    \includegraphics[width=1\textwidth]{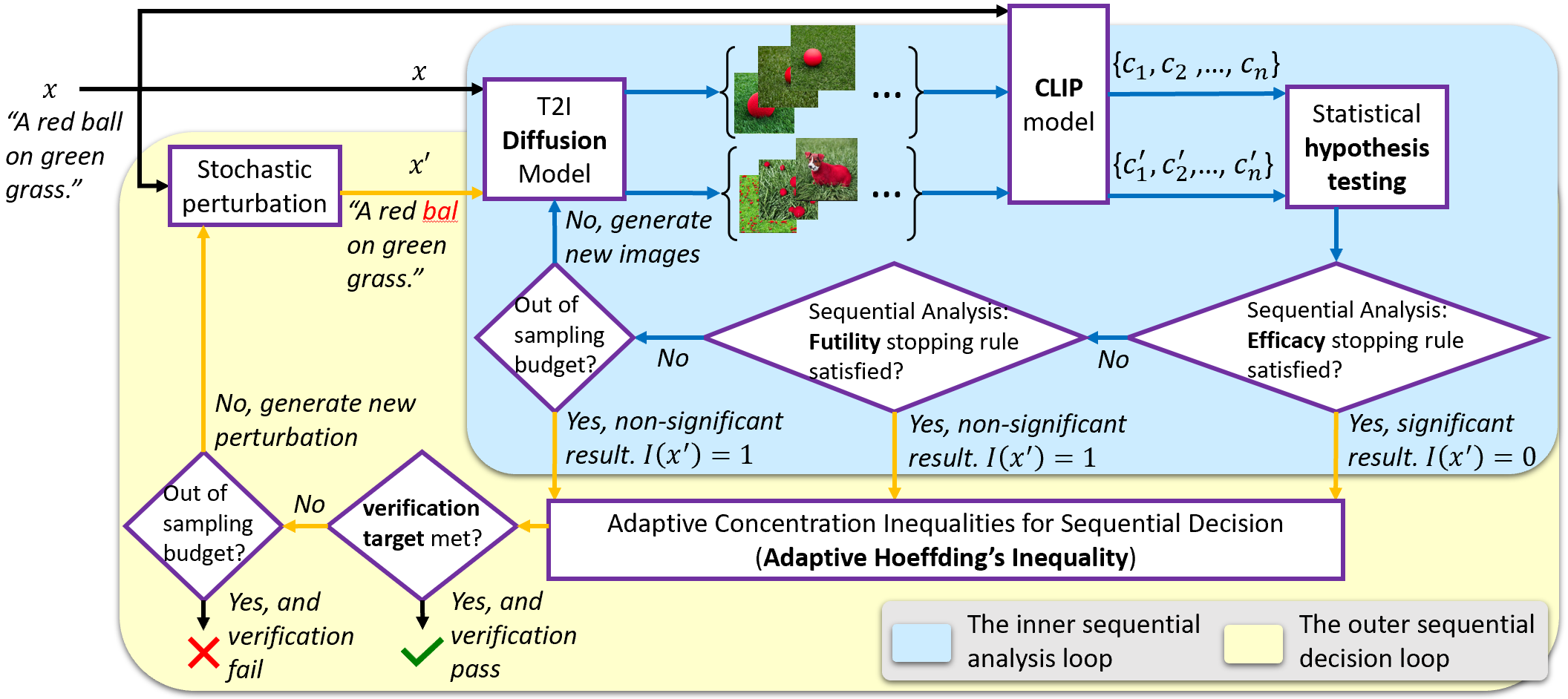}
    \caption{Workflow of ProTIP.}
    \label{fig_framework}
\end{figure}


\subsection{Stochastic Perturbation Generation}
\label{sec_pert_method}

While ProTIP can cope with any \textit{text perturbations generated stochastically}, WLOG, we study five \textit{character-level} text perturbation methods from \cite{morris2020textattack}. 
First, we set the perturbation rate like 10\%, indicating that 10\% of the words in the sentence will be perturbed. 
Then, one of the five perturbation methods (insert, substitute, swap, delete, and keyboard error) shown in Table~\ref{tab_perturbations} of Appendix~\ref{sec_app_Stochastic_Perturbation_Generation} will be randomly selected. Using insertion perturbation as an example, a letter from the 26 English alphabets is randomly selected and inserted into a randomly chosen position within a word in the sentence. This process is repeated until the perturbed sentence reaches a 10\% perturbation rate, ensuring it's not a duplicate before stopping. Taking ``\textit{A white dog plays with a red ball on the green grass}'' as an example, after insertion perturbation, we get the perturbed sentence ``\textit{A white \textcolor{red}{daog} plays with a red ball on the green grass}''. More details and examples about each text perturbation method can be found in Table~\ref{tab_perturbations} of Appendix~B


As Def.~\ref{def_Prob_rob_classification}, the intuition behind robustness suggests that the outputs of a DL model should remain invariant to perturbations on the inputs, while the inputs, before and after the perturbation, should be deemed very similar---thus, Def.~\ref{def_Prob_rob_classification} restricts the perturbation in a norm-ball $B(x,\gamma)$ with a small radius $\gamma$. Such perturbation distance control is easily achievable in Computer Vision where the inputs (either raw pixel values or the latent space feature values) can be normalised and projected into a continuous input space. However, for DL models take text inputs (like T2I DMs), the discrete nature of text make it more challenging to study robustness \cite{li2023perturbscore} as a single character change may completely alter the semantics meaning.
E.g., ``a red \textcolor{red}{ball} on ...'' and ``a red \textcolor{red}{bell} on ...'' have totally different semantic meanings with only one character perturbed. Thus, to study the robustness of T2I DMs, first we need to ensure that the perturbed text \(x'\) has the same semantics to the original text \(x\). Following the idea from \cite{gao2023evaluating,zhuang2023pilot,zhang2024revealing, liu2023riatig, struppek2023rickrolling, zhai2023text} for the same problem, we employ CLIP \cite{radford2021learning} as the model for text encoding, leveraging CLIP's ability to represent image and text information while preserving their relationships.


Therefore, as per Def.~\ref{def_Prob_rob_t2i}, we use CLIP's text encoder \(T\) to extract the original embeddings $T(x)$ and perturbed sentence embeddings $T(x')$ to compute the similarity score $s(T(x),T(x'))$ (e.g., by cosine similarity). Only when the similarity score is greater than a given threshold $\gamma$, we deem $x'$ is a valid perturbation, where $\gamma$ can be determined by the method developed in \cite{li2023perturbscore} (which yields the maximum perturbation distance in the continuous embedding space that preserves the semantics meaning for a given text). In this way, we ensure that only valid perturbations (sharing similar semantics as the original text) are used for robustness evaluation in ProTIP.


\subsection{Indication of Adversarial Examples}
\label{sec_trial_metric}

For a given perturbed input $x'$, whether it is an AE depends on if the T2I DM output distribution $Pr(Y \mid X=x')$ is significantly different from $Pr(Y \mid X=x)$. In this section, we describe the details of how ProTIP solves the question as a sequential, two-sample statistical testing problem, shown as the ``inner loop'' in Fig.~\ref{fig_framework}. 
Although there are techniques capable of \textit{measuring} the distance between two given distributions, e.g., the KL divergence \cite{KLdivergence} and the Maximum Mean Discrepancy \cite{gretton2012kernel}, they are not a \textit{test} and thus they do not provide a statistically principled way to determine \textit{whether or not} the distributions are different.

\subsubsection{Statistical hypothesis testing for indicating AE.}

First, a set of images is generated for the perturbed input $x'$, denoted as $\{y'_1,\dots, y'_n \}$, i.e., a set of independent and identically distributed (i.i.d.) samples from the distribution $Pr(Y \mid X=x')$. Off-the-shelf two-sample test tools do not perform well for high-dimensional data like images \cite{gretton2012kernel}. Thus, similar to \cite{gao2023evaluating}, we adapt CLIP Score to evaluate the correlation between a generated caption for an image and the actual content of the image, which has been found to be highly correlated with human judgement \cite{hessel-etal-2021-clipscore}. The metric is formally defined as  \cite{hessel-etal-2021-clipscore}:
\begin{equation}
    \label{eq_clip}
    CLIPScore(x,y):=  max(100 \cdot cos(T(x), V(y)),0)
\end{equation} 
which corresponds to the cosine similarity between visual CLIP embedding $V(y)$ for an image $y$ and textual CLIP embedding $T(x)$ for a caption $x$. We refer readers to Appendix B for details on the CLIP Score.
For each image $y'_i$, we calculate a set of i.i.d. ``CLIP scores'' $\{c'_1, \dots, c'_n\}$ by measuring $c'_i:=  CLIPScore(x,y_i')$. A CLIP score $c'_i$ represents the \textit{coexistence likelihood} of the image $y'_i$ and the original text $x$.
We repeat the same calculation for each image $y_i$ generated from the original text $x$, collectively forming another set of CLIP scores $\{c_1, \dots, c_n\}$.

Intuitively, a higher $c'_i$ indicates a ``less adversarial'' $x'$ \cite{gao2023evaluating,Du2023StableDI}. Thus, if $x'$ is not an AE, then the statistics on the group of $c'_i$s should not be significantly smaller than the group of $c_i$s, for which we can do an \textit{one-sided statistical hypothesis testing} with the following null and alternative hypotheses:
\begin{itemize}
    \item $\mathcal{H}_0$: There is no difference in the two groups of CLIP scores.
    \item $\mathcal{H}_1$: The CLIP scores of the $c'_i$ group is smaller than the $c_i$ group.
\end{itemize}
While there are established statistical testing methods, e.g., the t-test, u-test, and their variants \cite{lakens_2022_imp}, that can be employed to answer the above question, their applicability depends on whether the data distribution meets certain assumptions. Our ProTIP is compatible with any applicable statistical tests in this step, cf. later Sec.~\ref{sec_experiments} for our choice in the experiments. We can now implement the AE indicator function in Def.~\ref{def_Prob_rob_t2i} as:
\begin{equation}
\label{eq_hyp_test}
\begin{split}
& I{(x')} = \begin{cases}
\begin{aligned}
&1, \quad \text{accept the null hypothesis } \mathcal{H}_0\\
&0, \quad \text{reject the null hypothesis } \mathcal{H}_0
\end{aligned}
\end{cases}
\end{split}
\end{equation}

\subsubsection{Sequential analysis with early stopping rules.}
\label{sec_efficiency_eval}

In contrast to conventional DL tasks like classification, where detecting an AE involves an ``one-off'' test to see if the prediction label changes, determining if a perturbation affects the T2I DM requires \textit{multiple generations} and comparing the \textit{distributional differences} between two sets of images generated before and after the perturbation. This shift in the evaluation approach leads to a significantly increased computational workload. 
That said, we employ sequential analysis with early stopping rules \cite{jennison1999group, proschan2006statistical, wassmer2016group} for the hypothetical test in Eq.~\eqref{eq_hyp_test}. Instead of collecting all data with the maximum sample size (fixed at the design time) and then analysing the data \textit{a single time by the end}, sequential analysis conducts \textit{interim} analyses during data collection, so that we can prematurely stop data collection at an interim analysis upon rejecting or accepting the null hypothesis. Thanks to early stopping, sequential designs will, on average, require fewer samples \cite[Chap.~10]{lakens_2022_imp}.

In ProTIP, we adopt sequential analysis with both \textit{Efficacy} and \textit{Futility} stopping rules, by controlling the $\alpha$ (Type I error, false positive when a null hypothesis is incorrectly rejected) and the $\beta$ (Type II error, false negative when the null hypothesis is accepted and it is actually false), respectively. Intuitively, the Efficacy stopping rule says, if the analysis reveals a statistically significant result at some interim stage, data collection can be terminated because we reject $\mathcal{H}_0$ when observing the p-value $p<\alpha_k$ (where $\alpha_k$ is the cumulative Type I error rate at the $k$-th interim analysis, controlled by alpha-spending functions \cite{lakens_2022_imp}). On the other hand, we may also stop for futility, which means it is either impossible or very unlikely for the final analysis to yield $p < \alpha$. This can be implemented by controlling the Type II error rate across interim analysis using a beta-spending function. We refer readers to \cite[Chap.~10]{lakens_2022_imp} for more details on the topic.

Specifically, we employ the \txt{R} package \textsc{rpact} \cite{rpact} to perform the sequential analysis by setting Pocock type alpha/beta spending functions \cite{gordon1983discrete} 
to determine the sample size and the thresholds for efficacy/futility early stopping rules. Details are presented in Sec.~\ref{sec_experiments} and Appendix~D

\subsection{Decision-making for Verification}
\label{sec_evaluation_pr}

While we establish a method of indicating AEs for \textit{a given perturbation} $x'$ in the last section, to make the decision if a given verification target (cf. Def.~\ref{def_verfication}) is met we need to estimate $R_M$ over \textit{a population of perturbations}, for which we propose the ``outer loop'' in Fig.~\ref{fig_framework} and explain details in this section.

Concentration inequalities \cite{boucheronconcentration}, such as Chernoff inequality, Azuma's bound, and Hoeffding's inequality, represent important statistical methods widely employed in ensuring dependable decision-making with probabilistic guarantees. Specifically, in probability theory, Hoeffding's inequality \cite{hoeffding1994probability} provides a bound on the probability that the sum of bounded independent random variables deviates from its expected value by more than a certain amount. In ProTIP, for an input $x$, the T2I DM's probabilistic robustness $R_M(x)$ is the proportion of AEs over a population of all possible perturbations on $x$. Normally this population is very large (if not infinite)\footnote{The population size depends on the length of the original text $x$ and perturbation parameters like perturbation rate. Note, ProTIP can also cope with extreme (and simpler) cases in which the perturbation population is small and its members can be enumerated, e.g., when $x$ is a single word with one character to be changed.}, thus the ground truth of $R_M(x)$ can only be estimated as the sample mean of samples drawn from the population. Irrespective of the sample size, there will inevitably be discrepancies between the sample mean and the ground truth population mean. Nevertheless, Hoeffding's inequality allows us to establish tight probabilistic bounds on this error (cf. Appendix~C
for more on Hoeffding's inequality). A limitation of Hoeffding's inequality is its requirement for a \textit{predetermined, process-independent} sample size. However, in most circumstances, we generally have no idea how many samples are sufficient to verify the model, a priori. Thus, it is common to allocate a maximised sample size that accommodates the budget limit which may result in unnecessary waste.

Inspired by \cite{zhao2016adaptive,zhang2022proa}, our ProTIP employs an ``adaptive'' version of Hoeffding's inequality, in which the sample size itself is a variable. This enables the sampling process to stop as soon as the ``just right'' number of perturbations has been tested for making the verification decision. Moreover, the following theorem proves that the theoretical guarantee on the estimation errors is preserved:
\begin{theorem}[Adaptive Hoeffding's Inequality]
\label{thm_adp_hoeff}
    We know \(I(x'_i)\) is a binary 0--1
    random variable. Let $\hat{\mu}_I^{(n)}=\frac{1}{n}\sum_{i=1}^{n} I(x'_i)$ (i.e., the sample mean). Also let $J$ be a random variable on $\mathbb{N} \cup \{\infty\}$, and 
     $\varepsilon(\sigma, n) = \sqrt{\frac{0.6 \cdot \log(\log_{1.1}n+1) + 1.8^{-1} \cdot \log\left(24/{\sigma}\right)}{n}}$, then we have:
\begin{equation}
\label{eq_adp_hoeff}
Pr\left( | \hat{\mu}_I^{(J)} - R_M| \leq \varepsilon(\sigma,J) \right) \geq 1-\sigma
\end{equation}
where $R_M$ is the true population mean of $I(x'_i)$, and $\sigma$ is a given confidence level.
\end{theorem}
The proof for Theorem~\ref{thm_adp_hoeff} is presented in \cite{zhang2022proa}, which adapts the more general proof in \cite{zhao2016adaptive} for the binary 0--1 variables $I_i$ and also rearranges terms to align with the form of the original Hoeffding's inequality. We further prove two corollaries in Appendix~C
to better explain our later experimental results: One compares the tightness of the two bounds derived from the original and adaptive Hoeffding's inequality; the other concerns their monotonicity with respect to $n$ and $\sigma$.

In ProTIP, we directly apply Eq.~\eqref{eq_adp_hoeff} by sequentially increasing the perturbation number $J$ from 1 to $j_{max}$ (the maximum sampling budget). Given the verification target $(b_l,\sigma)$ (cf. Def.~\ref{def_verfication}), whenever the \textit{empirically estimated lower bound} on $R_M(x)$ yielded by Eq.~\eqref{eq_adp_hoeff} (after rearranging the inequality) is greater or equal to the specified requirement $b_l$ (i.e., $\hat{\mu}_I^{(J)}-\varepsilon(\sigma,J)\geq b_l$), we stop generating new perturbations and assert ``pass'' for the verification. Otherwise, when $J=j_{max}$, ProTIP asserts ``fail''  with an estimated robustness.




\section{Experiments}
\label{sec_experiments}


In our experiments, we study three versions of the widely acclaimed and open-source\footnote{We run experiments locally on our own servers for efficiency, as ProTIP generates a large number of queries to the model under verification. Thus, we exclude experiments on non-open source, commercial models, e.g., DALLE-3 and Midjourney.} Stable Diffusion (SD) model \cite{rombach2022high}, SD-V1.5, SD-V1.4 and SDXL-Turbo. While different versions of SD models have the same structure, the higher version is further trained based on the previous version; and SDXL-Turbo is based on Adversarial Diffusion Distillation \cite{sauer2023adversarial}.
We use the MS-COCO 
dataset \cite{lin2014microsoft}, a comprehensive collection designed for common computer vision tasks including captioning. The dataset comprises 328,000 images with captions, from which we randomly select captions as prompts for the T2I DM and then apply the stochastic perturbation method discussed in Sec.~\ref{sec_pert_method} to generate perturbed inputs.

All experiments are conducted on NVIDIA GeForce RTX 3090 GPUs, Python 3.11, PyTorch 2.1.1. We use the off-the-shelf \txt{R} package \cite{Rpackage} to design group sequential analysis parameters, cf.~Appendix~D
for details. All our models, data, source code and experimental results are publicly accessible at our Github site.

\subsection{Effectiveness of ProTIP}

In order to evaluate the effectiveness of ProTIP, i.e., how accurate our ProTIP is, we need to know the ground truth probabilistic robustness $R_M(x)$ for the given prompt $x$. However, as articulated in Sec.~\ref{sec:problem_statement}, the ground truth of $R_M(x)$ can never be known (as it would require exhaustively testing all possible perturbations). Consequently, we can only \textit{approximate the ground truth} by using a significantly larger number of samples than what would normally be used in ProTIP to demonstrate accuracy, which is arguably a common practice in statistical inference, e.g., \cite{webb2018statistical,Huang_2023_ICCV,dong_reliability_2023}. Specifically, we conduct the following two steps: \textit{i)} estimate the ground truth of $I(x')$ for a given perturbation $x'$ by generating a large number of images; \textit{ii)} then estimate the ground truth $R_M(x)$ using a large number of $x'$s with their ground truth $I(x')$. For step 1, we stop generating the images for the $x'$ until the distribution of its CLIP scores ``converge'' (i.e., the shape of the CLIP score distribution shows no significant changes when new images are generated). We then check whether the two CLIP score distributions follow a normal distribution. If yes, we execute a t-test, otherwise a u-test. In step 2, we generate a fixed number of $1,000$ perturbations, which is much larger than the adaptive number of perturbations (around $50\!\sim\!350$) used in ProTIP. Then the original Hoffeding's inequality (which yields a tighter error bound than the adaptive version used in ProTIP for the same number of perturbations, cf.~Corollary \ref{col_hi_tightness}) is applied to approximate the ground truth $R_M(x)$.


As shown in Fig.~\ref{fig:effectiveness}, for each given input \( x \), we conduct 3 sets of comparative experiments for different perturbation rates and different version of SD models. In each experiment, we run ProTIP for 3 different confidence levels $1-\sigma$ (coloured solid lines), and set the verification target to 0.8 (solid black line), while the dashed red line represents the (approximated) ground truth. ProTIP would stop whenever the verification target is met (i.e., the intersection point of the coloured and black lines). But for illustration, we also plot the assessment result after the intersection point, represented by dotted lines.

For a given prompt, Fig.~\ref{fig:effectiveness}(a) illustrates its robustness estimation with 10\% perturbation rate, where the estimates based on different confidence levels all converge to the ground truth. The noticeable gap between the approximated ground truth and our results is expected since ProTIP focuses on the lower bound estimation (cf. Eq.~\eqref{eq_verf_target}), being conservative. Such gap can be reduced when more perturbations are generated for a more accurate estimation (cf.~Corollary \ref{col_monotonicity_bounds}).
For the case of Fig.~\ref{fig:effectiveness}(a), we observe that ProTIP only requires $70 \!\sim\!100$ perturbations to achieve the given verification target $(0.8, \sigma)$.

Fig.~\ref{fig:effectiveness}(b) presents the results with an increased perturbation rate of 20\%. As expected, the ground truth robustness is lower, as well as the ProTIP results---it asserts ``fail'' reflects that the ground truth is lower than the verification target.
In Fig.~\ref{fig:effectiveness}(c), ProTIP is applied to a older version SD V1.4 model. Compared to (a), it is evident that, under the same 10\% perturbation rate, the estimated robustness of SD V1.4 are lower than those of SD V1.5, implying its weaker ability to withstand adversarial perturbations, which conforms to the observation in \cite{zhang2023robustness, lee2023holistic, wang2023the}. While Fig.~\ref{fig:effectiveness} only shows the results of a single prompt, results for more prompts and SDXL-Turbo are presented in Appendix~E
\begin{figure}
    \centering
    \includegraphics[width=\textwidth]{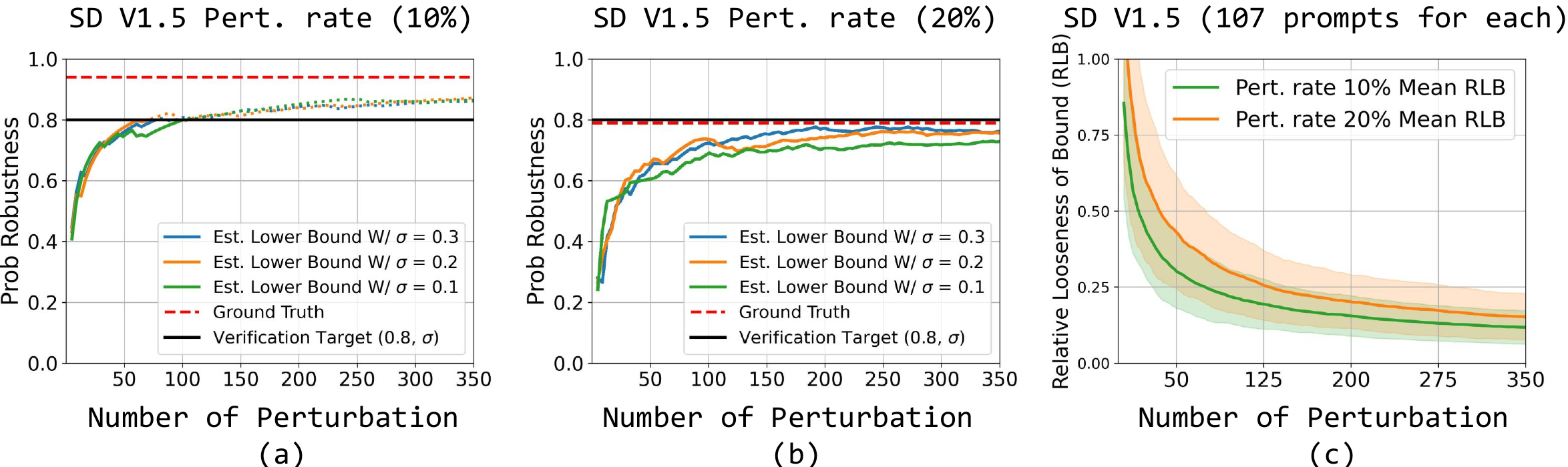}
    \caption{ProTIP results for a given prompt, with different confidence levels $1-\sigma$ (a)(b); Mean \& Std. (shared area) of RLB over 107 prompts (c).}
    \label{fig:effectiveness}
\end{figure}

\subsection{Efficiency of ProTIP}

The efficiency of ProTIP manifests firstly in identifying AEs---the ``inner loop'' of Fig.~\ref{fig_framework}. 
As the aforementioned sequential statistical testing method, following the parameter settings of Appendix~D,
WLOG we divide it into 5 stages to conduct interim analysis. 
Fig.~\ref{fig:Efficiency}(a) illustrates the number of perturbations identified as Non-AEs/AEs at each interim analysis stage, for a given prompt processed by SD V1.5 with perturbation rates of 10\% and 20\%, as well as by SD V1.4 with a 10\% perturbation rate. 
For V1.5 with 10\% rate, we observe that, out of the total 400 perturbations, 347 are identified as Non-AEs. Thanks to the futility stopping rule, 246 of these Non-AEs can be identified at stage 1, saving 4 times the cost of generating images compared to what would be required without early stopping rules (i.e., test at the final stage 5). Considering other stages as well, in total it reduces the computational cost by detecting 97\% of Non-AEs before reaching the final stage 5. Similarly, 68\% of AEs are detected early due to the efficacy stopping rule. With a 20\% perturbation rate and an older version, the total number of AEs increased and Non-AEs decreased, yet overall, the effect of two early stopping rules in the statistical testing for indicating Non-AEs/AEs remains evident. While Fig.~\ref{fig:Efficiency}(a) demonstrates the results for only one prompt, statistics for 36 randomly selected prompts are presented in Table \ref{tab_efficiency}.

The other efficient aspect of ProTIP is demonstrated by introducing the adaptive concentration inequalities to dynamically determine the number of perturbations, cf.~the ``outer loop'' of Fig.~\ref{fig_framework}. To compare with the original Hoeffding's inequality which requires a predetermined, process-independent sample size, we have to do ``what if'' calculations by assuming the sample size used by it. Fig.~\ref{fig:Efficiency}(b) presents the ProTIP results and the (original) Hoeffding's inequality with sample sizes of 50, 100, and 200 (dotted cross in green). With sample size 50, Hoeffding's inequality asserts an incorrect verification result ``fail'', while ProTIP correctly verifies it, although with more samples of 78. This is non-surprising as the error bound is bigger when with limited samples (cf.~Corollary \ref{col_monotonicity_bounds}). Thus no practitioners would apply Hoeffding's inequality with such small sample size, rather allocate a much larger sample size for a small estimation error.
The case with sample size 200 is replicating this scenario, in which both methods make the correct verification decision while ProTIP saves 122 perturbations. 
For the case of sample size 100, indeed both methods yield similar results (and Hoeffding's inequality is slightly better, cf.~Corollary \ref{col_hi_tightness}) with similar numbers of samples. However, in practice, we never know such ``just right'' sample size a priori when applying Hoeffding's inequality, highlighting ProTIP's superiority in efficiency.
\begin{figure}
    \centering
    \includegraphics[width=\textwidth]{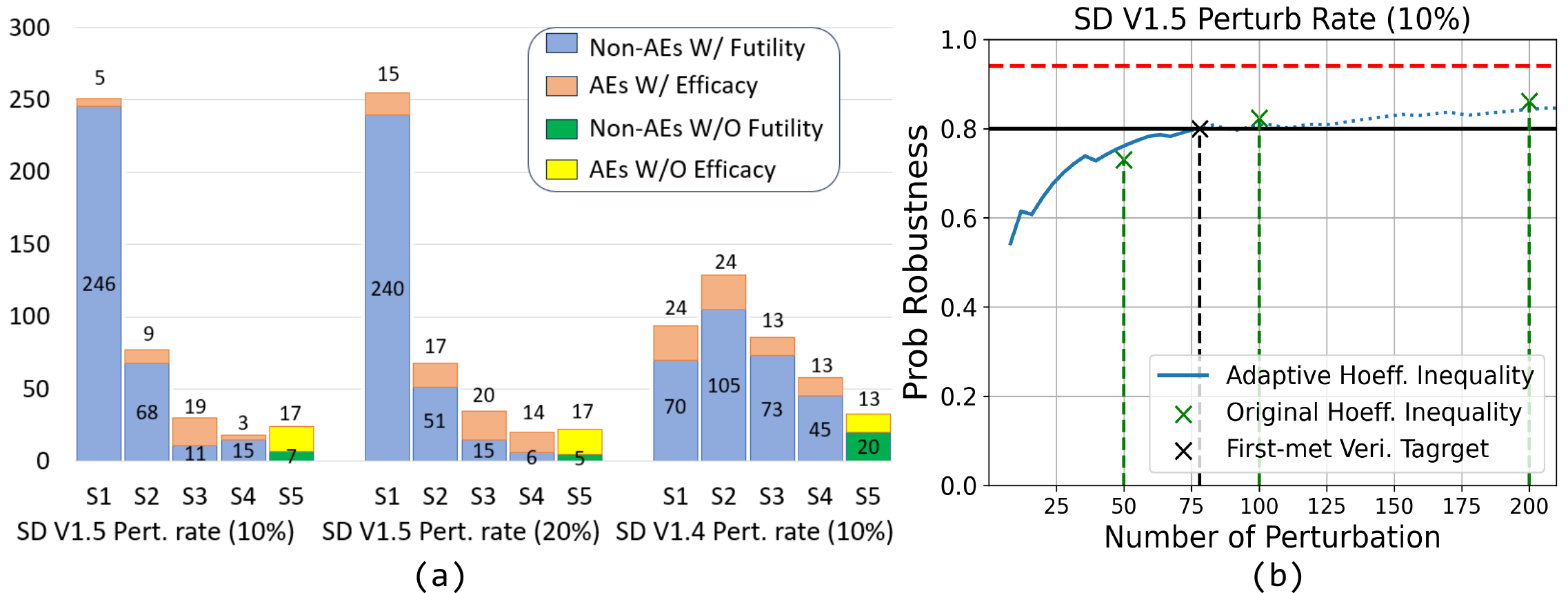}
    \caption{(a) Number of perturbations (out of 400) identified as Non-AEs/AEs at each interim stage (S) in the sequential hypothesis testing. (b) ProTIP with the (adaptive) sample size of 78 vs. Hoeffding's inequality with fixed sample size 50, 100, and 200.}
    \label{fig:Efficiency}
\end{figure}

\begin{table}[!ht]
    \centering
    \caption{The mean and std. of perturbations identified as Non-AE/AE by the two early stopping rules before the final stage 5, over 36 randomly selected prompts.}
    \label{tab_efficiency}
    \resizebox{0.95\textwidth}{!}{%
    \begin{tabular}{|c|c|c|c|c|c|c|c|c|c|}
    \hline
        \multirow{2}{*}{Model} & \multirow{2}{*}{Pert. rate} & \multicolumn{2}{c|}{Stage 1} & \multicolumn{2}{c|}{Stage 2} & \multicolumn{2}{c|}{Stage 3} & \multicolumn{2}{c|}{Stage 4} \\
        \cline{3-10}
        & & Efficacy & Futility & Efficacy & Futility & Efficacy & Futility & Efficacy & Futility \\
        \hline
        \multirow{2}{*}{SD V1.5} & 10\% & $47 \pm 32$ & $121 \pm 70$ & $22 \pm 16$ & $82 \pm 35$ & $20 \pm 14$ & $43 \pm 28$ & $14 \pm 12$ & $20 \pm 14$ \\ 
        \cline{2-10}
        & 20\% & $83 \pm 48$ & $99 \pm 67$ & $38 \pm 20$ & $54 \pm 26$ & $23 \pm 15$ & $34 \pm 19$ & $19 \pm 9$ & $19 \pm 15$ \\ 
        \hline
        SD V1.4 & 10\% & $48 \pm 29$ & $100 \pm 66$ & $25 \pm 16$ & $86 \pm 38$ & $19 \pm 17$ & $45 \pm 29$ & $16 \pm 10$ & $23 \pm 18$ \\ 
        \hline
        SDXL Turbo & 10\% & $64 \pm 44$ & $136 \pm 76$ & $21 \pm 14$ & $73 \pm 32$ & $15 \pm 13$ & $34 \pm 23$ & $12 \pm 8$ & $19 \pm 14$ \\ 
        \hline
    \end{tabular}%
    }
\end{table}

\subsection{Application of ProTIP for Ranking Defence Methods}

As a robustness assessment tool, a natural use case of our ProTIP is to rank common defence methods for T2I DMs. For character-level perturbations that are perceivable and semantic, scrutinising the input is a direct and universally applicable defence. Thus, we study three commonly used misspelling checking tools: Python Autocorrect 0.3.0 \cite{gao2018black}, Pyspellchecker \cite{pyspellchecker}, and Gramformer \cite{Gramformer} for error correction on the perturbed inputs before inputting them into T2I DMs.

In Fig.~\ref{fig:defence}, for an example prompt, all three defence tools may improve the model's resistance to perturbations in inputs, compared to the case without using any defenders (the red curve). Notably, Autocorrect demonstrates superior overall performance compared to the other two tools. Such observations are further confirmed by the box plots in Fig.~\ref{fig_defence_ranking} over a set of 36 randomly selected prompts, ranking their performance. Cf. Appendix E
for results of more prompts/models.
\begin{figure}[h]
    \centering
    \includegraphics[width=\textwidth]{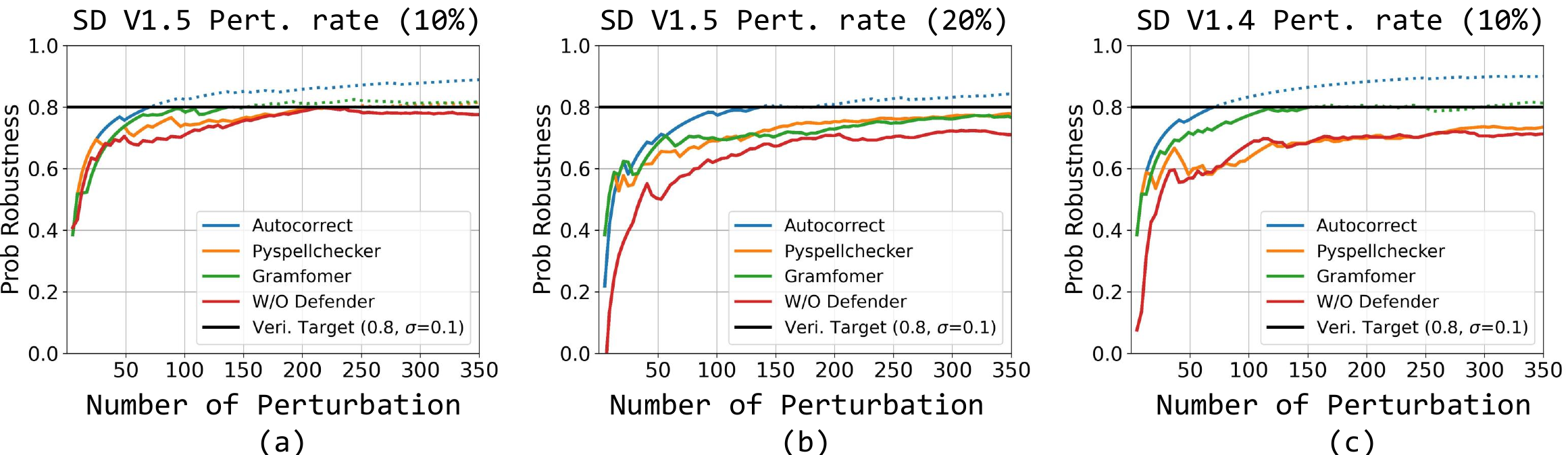}
    \caption{Probabilistic robustness estimated by ProTIP with/without defence methods.}
    \label{fig:defence}
\end{figure}


\section{Conclusion}
\begin{wrapfigure}{r}{0.35\textwidth}
    \centering
    \vspace{-2.2cm}
  \includegraphics[width=0.35\textwidth]{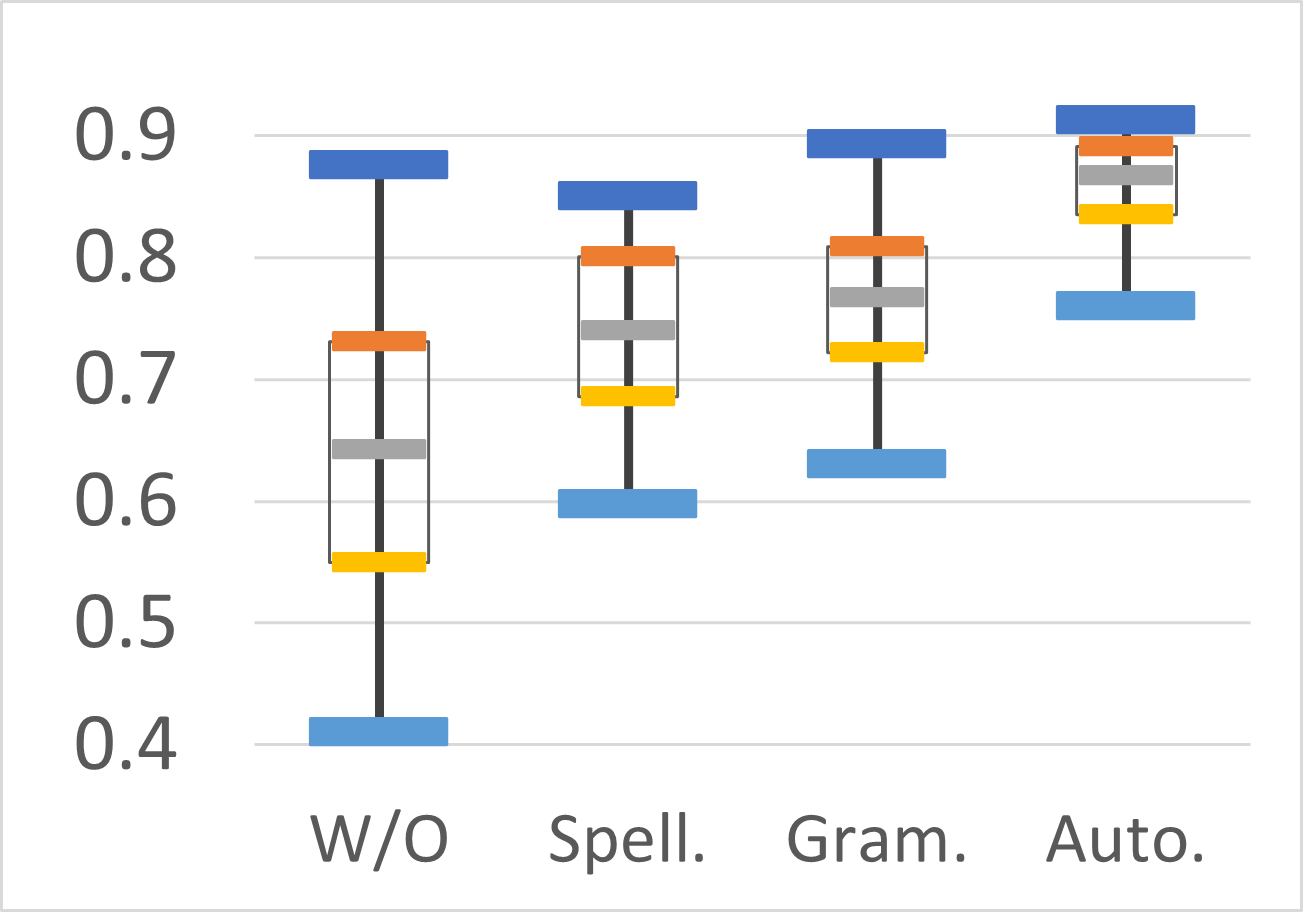}
    \caption{Box plot of the probabilistic robustness with \& without defenders, over 36 randomly selected prompts.}
    \vspace{-0.7cm}
    \label{fig_defence_ranking}
\end{wrapfigure}
In this work, for the first time, we formalise the definition of probabilistic robustness for T2I DMs and then establish an efficient framework, ProTIP, for evaluating it. As a black-box verification method, ProTIP is based on principled statistical inference approaches. It incorporates sequential analysis with early stopping rules in hypothesis testing when identifying AEs, along with adaptive concentration inequalities to dynamically adjust the number of stochastic perturbations needed to make the verification decision. Empirical experiments have substantiated the effectiveness and efficiency of ProTIP. Finally, we demonstrate a use case of ProTIP to rank various commonly used defence methods, highlighting its versatility and applicability.

\section*{Acknowledgments} This work is supported by the British Academy's Pump Priming Collaboration between UK and EU partners 2024 Programme [PPHE24/100115]. We thank Hongchen Gao for his valuable comments on our work.

%
%
\bibliographystyle{splncs04}
\bibliography{ref}

\clearpage

\appendix

\section{Denoising Diffusion Probabilistic Models}
\label{sec_app_text2image_and_per_method}

Denoising Diffusion Probabilistic Models (DDPMs)\cite{ho2020denoising} \cite{koller2009probabilistic} \cite{sohl2015deep} are a class of probabilistic generative models that apply a noise injection process, followed by a reverse procedure for sample generation. A DDPM is defined as two parameterized Markov chains: a forward chain that add random Gaussian noise to images to transform data distribution into a simple prior distribution and a reverse chain that convert the noised image back into target data by learning transition kernels parameterized by deep neural networks. 

\noindent\textbf{Forward diffusion process}: Given a data point sampled from a real data distribution $x_0 \sim q(x)$, a forward process begins with adding a small amount of Gaussian noise to the sample in $T$ steps, producing a sequence of noisy samples $x_1, \ldots, x_T$. The step sizes are controlled by a variance schedule $\{\beta_t \in (0,1)\}_{t=1}^T$.
\begin{equation}
\begin{aligned}
q(x_{1:T} | x_0) &:= \prod_{t=1}^T q(x_t | x_{t-1}), \\
q(x_t | x_{t-1}) &= \mathcal{N}\left(x_t; (1 - \beta_t) x_{t-1}, \beta_t I\right).
\end{aligned}
\end{equation}
The data sample $x_0$ gradually loses its distinguishable features as the step $t$ becomes larger. Eventually, when $T \to \infty$, $x_T$ is equivalent to an isotropic Gaussian distribution.

\noindent\textbf{Reverse diffusion process}:
The reverse process starts by first generating an unstructured noise vector from the prior distribution, then gradually removing noise by running a learnable Markov chain in the reverse time direction. Specifically, the reverse Markov chain is parameterized by a prior distribution $p(x_T) = \mathcal{N}(x_T; 0, I)$ and a learnable transition kernel $p_\theta(x_{t-1} | x_t)$. Therefore, we need to learn a model $p_\theta$ to approximate these conditional probabilities in order to run the reverse diffusion process.
\begin{equation}
\begin{aligned}
p_\theta(x_{0:T}) &= p(x_T) \prod_{t=1}^T p_\theta(x_{t-1} | x_t), \\
p_\theta(x_{t-1} | x_t) &= \mathcal{N}\left(x_{t-1}; \mu_\theta(x_t, t), \Sigma_\theta(x_t, t)\right),
\end{aligned}
\end{equation}
where $\theta$ denotes model parameters, often instantiated by architectures like UNet, which parameterize the mean $\mu_\theta(x_t, t)$ and variance $\Sigma_\theta(x_t, t)$. The UNet takes the noised data $x_t$ and time step $t$ as inputs and outputs the parameters of the normal distribution, thereby predicting the noise $\epsilon_\theta$ that the model needs to reverse the diffusion process. With this reverse Markov chain, we can generate a data sample $x_0$ by first sampling a noise vector $x_T \sim p(x_T)$, then successively sampling from the learnable transition kernel $x_{t-1} \sim p_\theta(x_{t-1} | x_t)$ until $t = 1$.

\section{Stochastic Perturbation Generation}
\label{sec_app_Stochastic_Perturbation_Generation}

\noindent\textbf{CLIP Score} We summarise the following reasons for choosing CLIP Score to measure the correlation between a generated caption for an image and the actual content of the image.
\begin{itemize}
    \item 1. While CLIP score is certainly not perfect as a metric (like other metrics) to mimic human-perception similarity, but some works show that CLIP is generally reliable and highly correlated with human judgement in user experiments, e.g.,\cite{hessel-etal-2021-clipscore}. Both CLIP scores are calculated (and thus anchored) by considering the text information, regardless of variations in image styles and backgrounds for additional evidence. Therefore, they are highly likely to be similar.
    \item 2. While the best metric for measuring T2I correctness remains an open question, CLIP score is commonly used in studying T2I robustness \cite{gao2023evaluating,zhuang2023pilot,zhang2024revealing, liu2023riatig, struppek2023rickrolling, zhai2023text}
    \item  3. While CLIP score is a building block of ProTIP, 
it can be substituted with other correctness metrics. The statistical models---the theoretical core and main contribution of ProTIP---are easily \textit{adaptable} to any new correctness metrics as future advancements.
\end{itemize}

\noindent\textbf{Stochastic Perturbation Method} Table \ref{tab_perturbations} for examples of stochastic text perturbations on T2I DM inputs.
\begin{table}[!h]
    \centering
    \caption{Examples of stochastic text perturbations}
    \label{tab_perturbations}
    \resizebox{\textwidth}{!}{%
    \begin{tabular}{p{2cm} p{6cm} p{6cm}}
    \toprule \toprule
        Perturbation & Description & Example \\ \midrule
        Insert & Insert a character randomly  & A white \textcolor{red}{daog} plays with a red ball on the green grass. \\ 
        Substitute & Substitute a character randomly & A white dog plays with a \textcolor{red}{rad} ball on the green grass.\\ 
        Swap & Swap two characters randomly & A white dog plays with a red ball on the green \textcolor{red}{garss}. \\ 
        Delete & Delete a character randomly & A white dog plays with a red \textcolor{red}{bll} on the green grass. \\ 
        Keyboard & Substitute a char.
        by keyboard distance  & A \textcolor{red}{whote} dog plays with a red ball on the green grass. \\ 
        \bottomrule
    \end{tabular}
    }
\end{table}

\section{Hoeffding's Inequality}
\label{sec_app_hi}

Reusing the notations in the main paper, let $I_1,\dots,I_n$ be i.i.d. samples drawn from a population, and \(R\) denotes the population mean to be estimated. Let $\hat{\mu}_I^{(n)}=\frac{1}{n}\sum_{i=1}^{n} I_i$, i.e., the sample mean, we have the following Hoeffding's Inequality \cite{hoeffding1994probability}.
\begin{theorem}[Hoeffding's Inequality]
\label{thm_hi}
\begin{equation}
\begin{aligned}
Pr \left( \left| \hat{\mu}_I^{(n)} - R \right| \leq \varepsilon \right) & > 1 - \sigma 
\end{aligned}
\end{equation}
where $\sigma = 2e^{-2n\varepsilon^2}$, equivalently, $\varepsilon = \sqrt{\frac{\log(2/\sigma)}{2n}}$. 
\end{theorem}

From the two Theorems \ref{thm_hi} and \ref{thm_adp_hoeff} regarding the original and adaptive Hoeffding's Inequality respectively, we may derive the following two corollaries.
\begin{corollary}[Tightness of the two bounds]
\label{col_hi_tightness}
Give a confidence level $1-\sigma$ and the same number of samples $n$, the estimation error (between the sample mean and population mean) $\varepsilon$ derived from the original Hoeffding's Inequality is always smaller than the adaptive Hoeffding's Inequality.
\end{corollary}

\begin{proof}
As Theorems \ref{thm_hi} and \ref{thm_adp_hoeff}, the estimation error $\varepsilon$ from the original and adaptive Hoeffding's Inequality are $\sqrt{\frac{\log(2/\sigma)}{2n}}$ and $\sqrt{\frac{0.6 \cdot \log(\log_{1.1}n+1) + 1.8^{-1} \cdot \log\left(24/{\sigma}\right)}{n}}$, respectively. We may prove the former is smaller than the latter either analytically or empirically. The inequalities of Corollary 1 can be written:

\begin{equation}
\begin{aligned}
\sqrt{\frac{\log(2/\sigma)}{2n}} < \sqrt{\frac{0.6 \cdot \log(\log_{1.1}n+1) + 1.8^{-1} \cdot \log\left(24/{\sigma}\right)}{n}} 
\end{aligned}
\end{equation}
Where $n\geq 1 \text{ and } \sigma \in (0, 1) $.
\begin{subequations}
\begin{align}
\frac{\log\left(2/\sigma\right)}{2n} &< \frac{0.6 \cdot \log\left(\log_{1.1} n + 1\right) + 1.8^{-1} \cdot \log\left(24/{\sigma}\right)}{n} \label{eq:line1} \\
\log(2/\sigma) &< 1.2 \cdot \log\left(\log_{1.1} n + 1\right) + \frac{10}{9} \cdot \log\left(24/\sigma\right) \label{eq:line2} \\
\log\left(2/\sigma\right) &< \log\left[\left(\log_{1.1}n + 1\right)^{1.2} \cdot \left(24/\sigma\right)^{\frac{10}{9}}\right] \label{eq:line3}
\end{align}
\end{subequations}
Because $n\geq 1$, hence, $\log_{1.1} n > 0$, $\log_{1.1} n + 1 > 1$, and $\left(\log_{1.1} n + 1\right)^{1.2} > 1$, the right side of  \eqref{eq:line3} is.
\begin{equation}
\begin{aligned}
\log\left[\left(\log_{1.1}n + 1\right)^{1.2} \cdot \left(24/\sigma\right)^{\frac{10}{9}}\right] > \log\left(24/\sigma\right)^{\frac{10}{9}} > \log\left(24/\sigma\right) 
\end{aligned}
\end{equation}
Hence, \eqref{eq:line3} will be: 
\begin{equation}
\begin{aligned}
\log\left(2/\sigma\right) < \log\left(24/\sigma\right) , \sigma \in (0, 1) 
\end{aligned}
\end{equation}
Therefore the equation holds, and the proof is complete.
\qed
\end{proof}

\begin{corollary}[Monotonicity to $n$ and $\sigma$]
\label{col_monotonicity_bounds}
For both the original and adaptive Hoeffding's Inequalities, the estimation errors (between the sample mean and population mean) $\varepsilon$ are monotonically decreasing to sample size $n$ (for n >= 2) and $\sigma$.
\end{corollary}
\begin{proof}
By taking the (partial) derivatives of the two analytical expressions of $\varepsilon$ r.w.t. $n$ and $\sigma$ respectively, we may establish negative results as what follows.
\begin{enumerate}
    \item $\sqrt{\frac{\log(2/\sigma)}{2n}}$ partial derivative with respect to \( n \): 
    \begin{align*}
    \frac{\partial \varepsilon(n, \sigma)}{\partial n} &= \frac{\partial}{\partial n} \left( \sqrt{\frac{\log(2/\sigma)}{2n}} \right) \\
    &= \frac{1}{2} \left( \frac{\log(2/\sigma)}{2n} \right)^{-\frac{1}{2}} \cdot \frac{\partial}{\partial n} \left( \frac{\log(2/\sigma)}{2n} \right) \\
    &= \frac{1}{2\sqrt{\frac{\log(2/\sigma)}{2n}}} \cdot \frac{\partial}{\partial n} \left( \frac{\log(2/\sigma)}{2n} \right) \\
    &= \frac{1}{2\sqrt{\frac{\log(2/\sigma)}{2n}}} \cdot \left( -\frac{\log(2/\sigma)}{2n^2} \right) \\
    &= -\frac{\log(2/\sigma)}{4n^2 \sqrt{\frac{\log(2/\sigma)}{2n}}}
    \end{align*}
    
    \item $\sqrt{\frac{\log(2/\sigma)}{2n}}$ partial derivative with respect to \( \sigma \):
    \begin{align*}
    \frac{\partial \varepsilon(n, \sigma)}{\partial \sigma} &= \frac{\partial}{\partial \sigma} \left( \sqrt{\frac{\log(2/\sigma)}{2n}} \right) \\
    &= \frac{1}{2} \left( \frac{\log(2/\sigma)}{2n} \right)^{-\frac{1}{2}} \cdot \frac{\partial}{\partial \sigma} \left( \frac{\log(2/\sigma)}{2n} \right) \\
    &= \frac{1}{2\sqrt{\frac{\log(2/\sigma)}{2n}}} \cdot \frac{\partial}{\partial \sigma} \left( \frac{\log(2/\sigma)}{2n} \right) \\
    &= \frac{1}{2\sqrt{\frac{\log(2/\sigma)}{2n}}} \cdot \left( -\frac{1}{2n\sigma \ln{10}} \right) \\
    &= -\frac{1}{4n\sigma \ln{10} \sqrt{\frac{\log(2/\sigma)}{2n}} }
    \end{align*}

    \item $\sqrt{\frac{0.6 \cdot \log(\log_{1.1}n+1) + 1.8^{-1} \cdot \log\left(24/{\sigma}\right)}{n}}$ partial derivative with respect to \( \sigma \):
    \begin{align*}
    \begin{split}
    \frac{\partial \varepsilon(n, \sigma)}{\partial \sigma} &= \frac{\partial}{\partial \sigma} \left( \sqrt{\frac{0.6 \cdot \log(\log_{1.1}n+1) + 1.8^{-1} \cdot \log\left(24/{\sigma}\right)}{n}} \right) \\
    &= \frac{1}{2} \left( \frac{0.6 \cdot \log(\log_{1.1}n+1) + 1.8^{-1} \cdot \log\left(24/{\sigma}\right)}{n} \right)^{-\frac{1}{2}} \\
    &\quad \times \frac{\partial}{\partial \sigma} \left( \frac{0.6 \cdot \log(\log_{1.1}n+1) + 1.8^{-1} \cdot \log\left(24/{\sigma}\right)}{n} \right) \\
    &= \frac{1}{2 \sqrt{\frac{0.6 \cdot \log(\log_{1.1}n+1) + 1.8^{-1} \cdot \log\left(24/{\sigma}\right)}{n}}} \\
    &\quad \times \frac{\partial}{\partial \sigma} \left( \frac{0.6 \cdot \log(\log_{1.1}n+1) + 1.8^{-1} \cdot \log\left(24/{\sigma}\right)}{n} \right) \\
    &= \frac{1}{2 \sqrt{\frac{0.6 \cdot \log(\log_{1.1}n+1) + 1.8^{-1} \cdot \log\left(24/{\sigma}\right)}{n}}} \\
    &\quad \times \left( -\frac{1}{1.8 \sigma n \ln{10}} \right) \\
    &= -\frac{1}{3.6 \sigma n \ln{10} \sqrt{\frac{0.6 \cdot \log(\log_{1.1}n+1) + 1.8^{-1} \cdot \log\left(24/{\sigma}\right)}{n}}}
    \end{split}
    \end{align*}

Therefore, the partial derivatives with respect to \( n \) and \( \sigma \) are both less than zero.

    \item $\sqrt{\frac{0.6 \cdot \log(\log_{1.1}n+1) + 1.8^{-1} \cdot \log\left(24/{\sigma}\right)}{n}}$ 
under the square root, we have $ f_1(n) + f_2(n, \sigma) $:
\[ f_1(n) = \frac{0.6 \cdot \log(\log_{1.1}n + 1)}{n} \]

\[ f_2(n, \sigma) = \frac{1.8^{-1} \cdot \log\left(\frac{24}{\sigma}\right)}{n} \]
$f_2(n, \sigma)$ partial derivative with respect to \( n \):
\begin{align*}
\frac{\partial}{\partial n} f_2(n, \sigma) &= \frac{\partial}{\partial n} \left( \frac{1.8^{-1} \cdot \log\left(\frac{24}{\sigma}\right)}{n} \right) \\
&= \frac{-1.8^{-1} \cdot \log\left(\frac{24}{\sigma}\right)}{n^2}
\end{align*}
Therefore, $f_2(n, \sigma)$ is monotonically decreasing.

The derivative of the denominator of \(f_1(n)\) is 1. For the numerator \(f(n) = 0.6 \cdot \log(\log_{1.1}n + 1)\),

\begin{align*}
f'(n) &= \frac{d}{dn} \left( 0.6 \cdot \log(\log_{1.1}n + 1) \right) \\
&= 0.6 \cdot \frac{1}{\log(1.1) \cdot (\log_{1.1} n + 1)} \cdot \frac{1}{n} \cdot \frac{1}{\ln(10)}
\end{align*}

$f'(n)$ is monotonically decreasing. When $n \geq 2$, $f'(n) \leq f'(2) = 0.165$.
Therefore, the derivative of the numerator is smaller than that of the denominator, and $f_1(n)$ is monotonically decreasing.
\end{enumerate}
\qed
\end{proof}

\section{Sequential Analysis Parameter Design}
\label{sec_app_r_package}

\subsection{Design parameters and output of group sequential design}

\begin{itemize}
    \item Type of design: Pocock type alpha spending
    \item Information rates: 0.200, 0.400, 0.600, 0.800, 1.000
    \item Significance level: 0.0500
    \item Type II error rate: 0.3000
    \item Type of beta spending: Pocock type beta spending
\end{itemize}

\textbf{Derived from User Defined Parameters}
\begin{itemize}
    \item Maximum number of stages: 5
    \item Stages: 1, 2, 3, 4, 5
\end{itemize}

\textbf{Default Parameters}
\begin{itemize}
    \item Two-sided power: FALSE
    \item Binding futility: FALSE
    \item Test: one-sided
    \item Tolerance: 1e-08
\end{itemize}

\textbf{Output}
\begin{itemize}
    \item Power: 0.1655, 0.3637, 0.5316, 0.6452, 0.7000
    \item Futility bounds (non-binding): -0.145, 0.511, 1.027, 1.497
    \item Cumulative alpha spending: 0.01477, 0.02616, 0.03543, 0.04324, 0.05000
    \item Cumulative beta spending: 0.08862, 0.15694, 0.21255, 0.25945, 0.30000
    \item Critical values: 2.176, 2.144, 2.113, 2.090, 2.071
    \item Stage levels (one-sided): 0.01477, 0.01603, 0.01729, 0.01833, 0.01918
\end{itemize}

\textbf{Group Sequential Design Characteristics}
\begin{itemize}
    \item Number of subjects fixed: 4.7057
    \item Shift: 7.2491
    \item Inflation factor: 1.5405
    \item Informations: 1.450, 2.900, 4.349, 5.799, 7.249
    \item Power: 0.1655, 0.3637, 0.5316, 0.6452, 0.7000
    \item Rejection probabilities under H1 : 0.16549, 0.19825, 0.16786, 0.11361, 0.05478
    \item Futility probabilities under H1 : 0.08862, 0.06832, 0.05561, 0.04690
    \item Ratio expected vs fixed sample size under H1 : 0.7938
    \item Ratio expected vs fixed sample size under a value between H0 and H1 : 0.7776
    \item Ratio expected vs fixed sample size under H0 : 0.5869
\end{itemize}

\subsection{Sample Size Calculation for a Continuous Endpoint}
Sequential analysis with a maximum of 5 looks (group sequential design), overall significance level 5\% (one-sided). The sample size was calculated for a two-sample t-test, $H_0: \mu(1) - \mu(2) = 0$, $H_1: \text{effect} = 0.5$, standard deviation = 1, power 70\%.

\begin{table}[ht!]
    \centering
    \caption{Group Sequential Design Characteristics}
    \begin{tabular}{lccccc}
        \toprule
        & \textbf{Stage 1} & \textbf{Stage 2} & \textbf{Stage 3} & \textbf{Stage 4} & \textbf{Stage 5} \\
        \midrule
        Information rate & 20\% & 40\% & 60\% & 80\% & 100\% \\
        Efficacy boundary (z-value scale) & 2.176 & 2.144 & 2.113 & 2.090 & 2.071 \\
        Futility boundary (z-value scale) & -0.145 & 0.511 & 1.027 & 1.497 & - \\
        Overall power & 0.1655 & 0.3637 & 0.5316 & 0.6452 & 0.7000 \\
        Expected number of subjects & \multicolumn{5}{c}{60.9} \\
        Number of subjects & 23.6 & 47.2 & 70.9 & 94.5 & 118.1 \\
        Cumulative alpha spent & 0.0148 & 0.0262 & 0.0354 & 0.0432 & 0.0500 \\
        Cumulative beta spent & 0.0886 & 0.1569 & 0.2126 & 0.2595 & 0.3000 \\
        One-sided local significance level & 0.0148 & 0.0160 & 0.0173 & 0.0183 & 0.0192 \\
        Efficacy boundary (t) & 0.959 & 0.644 & 0.512 & 0.436 & 0.385 \\
        Futility boundary (t) & -0.060 & 0.150 & 0.246 & 0.311 & - \\
        Overall exit probability (under H0) & 0.4570 & 0.2977 & 0.1526 & 0.0688 & - \\
        Overall exit probability (under H1) & 0.2541 & 0.2666 & 0.2235 & 0.1605 & - \\
        Exit probability for efficacy (under H0) & 0.0148 & 0.0113 & 0.0087 & 0.0062 & - \\
        Exit probability for efficacy (under H1) & 0.1655 & 0.1982 & 0.1679 & 0.1136 & - \\
        Exit probability for futility (under H0) & 0.4423 & 0.2864 & 0.1439 & 0.0626 & - \\
        Exit probability for futility (under H1) & 0.0886 & 0.0683 & 0.0556 & 0.0469 & - \\
        \bottomrule
    \end{tabular}
    \label{tab:group_sequential_design}
\end{table}

\subsection*{Legend}
\begin{itemize}
    \item (t): treatment effect scale
\end{itemize}

\subsection{Design Plan Parameters and Output for Means}

\subsection*{Design Parameters}
\begin{itemize}
    \item Information rates: 0.200, 0.400, 0.600, 0.800, 1.000
    \item Critical values: 2.176, 2.144, 2.113, 2.090, 2.071
    \item Futility bounds (non-binding): -0.145, 0.511, 1.027, 1.497
    \item Cumulative alpha spending: 0.01477, 0.02616, 0.03543, 0.04324, 0.05000
    \item Local one-sided significance levels: 0.01477, 0.01603, 0.01729, 0.01833, 0.01918
    \item Significance level: 0.0500
    \item Type II error rate: 0.3000
    \item Test: one-sided
\end{itemize}

\subsection*{User Defined Parameters}
\begin{itemize}
    \item Alternatives: 0.5
\end{itemize}

\subsection*{Default Parameters}
\begin{itemize}
    \item Mean ratio: FALSE
    \item Theta H0: 0
    \item Normal approximation: FALSE
    \item Standard deviation: 1
    \item Treatment groups: 2
    \item Planned allocation ratio: 1
\end{itemize}

\subsection*{Sample Size and Output}
\begin{itemize}
    \item Reject per stage [1]: 0.16549
    \item Reject per stage [2]: 0.19825
    \item Reject per stage [3]: 0.16786
    \item Reject per stage [4]: 0.11361
    \item Reject per stage [5]: 0.05478
    \item Overall futility stop: 0.2595
    \item Futility stop per stage [1]: 0.08862
    \item Futility stop per stage [2]: 0.06832
    \item Futility stop per stage [3]: 0.05561
    \item Futility stop per stage [4]: 0.04690
    \item Early stop: 0.9047
    \item Maximum number of subjects: 118.1
    \item Maximum number of subjects (1): 59.1
    \item Maximum number of subjects (2): 59.1
    \item Number of subjects [1]: 23.6
    \item Number of subjects [2]: 47.2
    \item Number of subjects [3]: 70.9
    \item Number of subjects [4]: 94.5
    \item Number of subjects [5]: 118.1
    \item Expected number of subjects under H0: 45
    \item Expected number of subjects under H0/H1: 59.6
    \item Expected number of subjects under H1: 60.9
    \item Critical values (treatment effect scale) [1]: 0.959
    \item Critical values (treatment effect scale) [2]: 0.644
    \item Critical values (treatment effect scale) [3]: 0.512
    \item Critical values (treatment effect scale) [4]: 0.436
    \item Critical values (treatment effect scale) [5]: 0.385
    \item Futility bounds (treatment effect scale) [1]: -0.0605
    \item Futility bounds (treatment effect scale) [2]: 0.1496
    \item Futility bounds (treatment effect scale) [3]: 0.2457
    \item Futility bounds (treatment effect scale) [4]: 0.3108
    \item Futility bounds (one-sided p-value scale) [1]: 0.55773
    \item Futility bounds (one-sided p-value scale) [2]: 0.30485
    \item Futility bounds (one-sided p-value scale) [3]: 0.15231
    \item Futility bounds (one-sided p-value scale) [4]: 0.06717
\end{itemize}

\subsection*{Legend}
\begin{itemize}
    \item (i): values of treatment arm i
    \item (k): values at stage k
\end{itemize}

\section{More Experimental Results}
\label{sec_app_more_results}

\begin{figure}[tb]
    \centering
    \includegraphics[width=\textwidth]{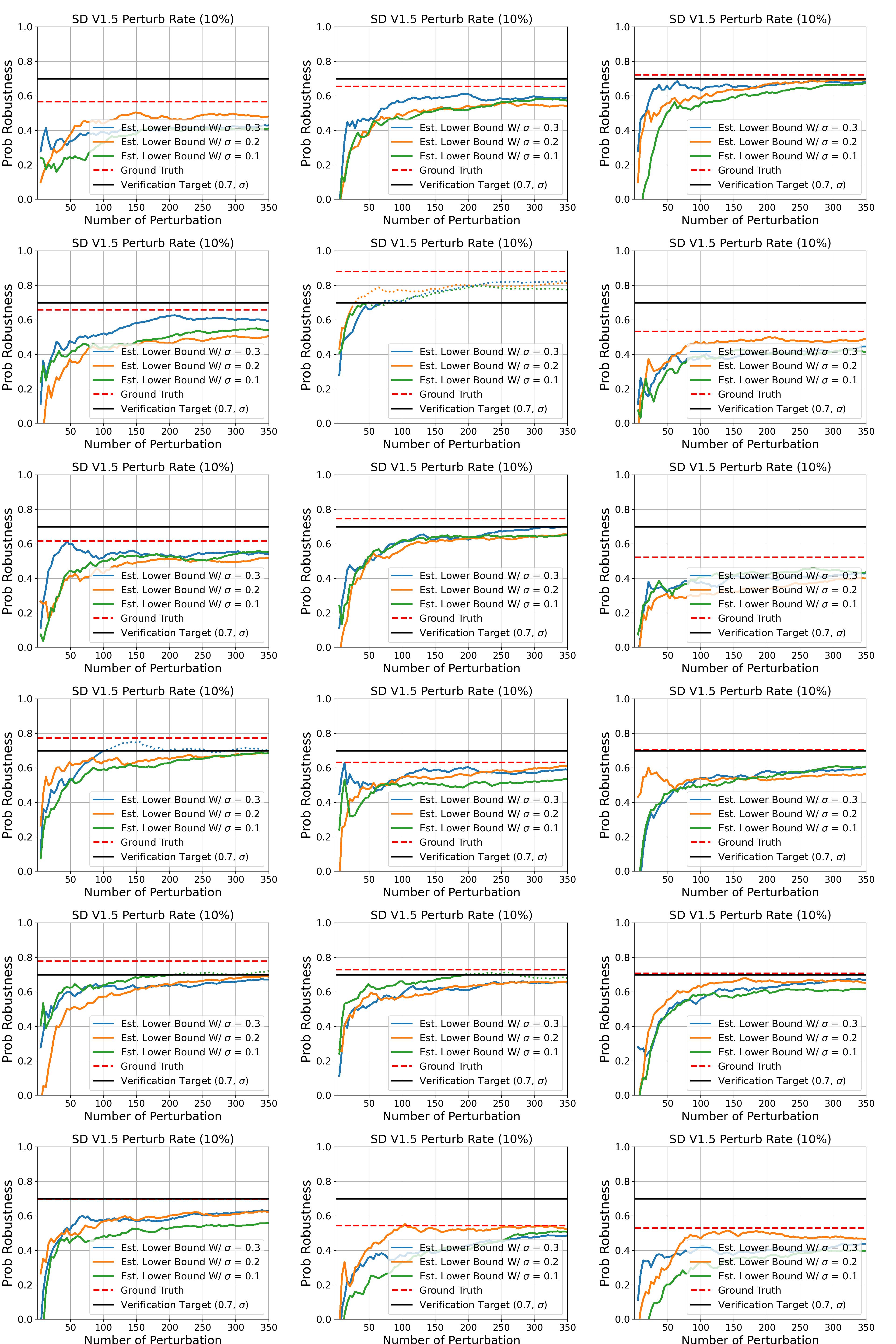}
    \caption{More results on ProTIP effectiveness for SD V1.5 Pert. Rate 10\% (Part 1).}
\end{figure}

\begin{figure}[tb]
    \ContinuedFloat
    \centering
    \includegraphics[width=\textwidth]{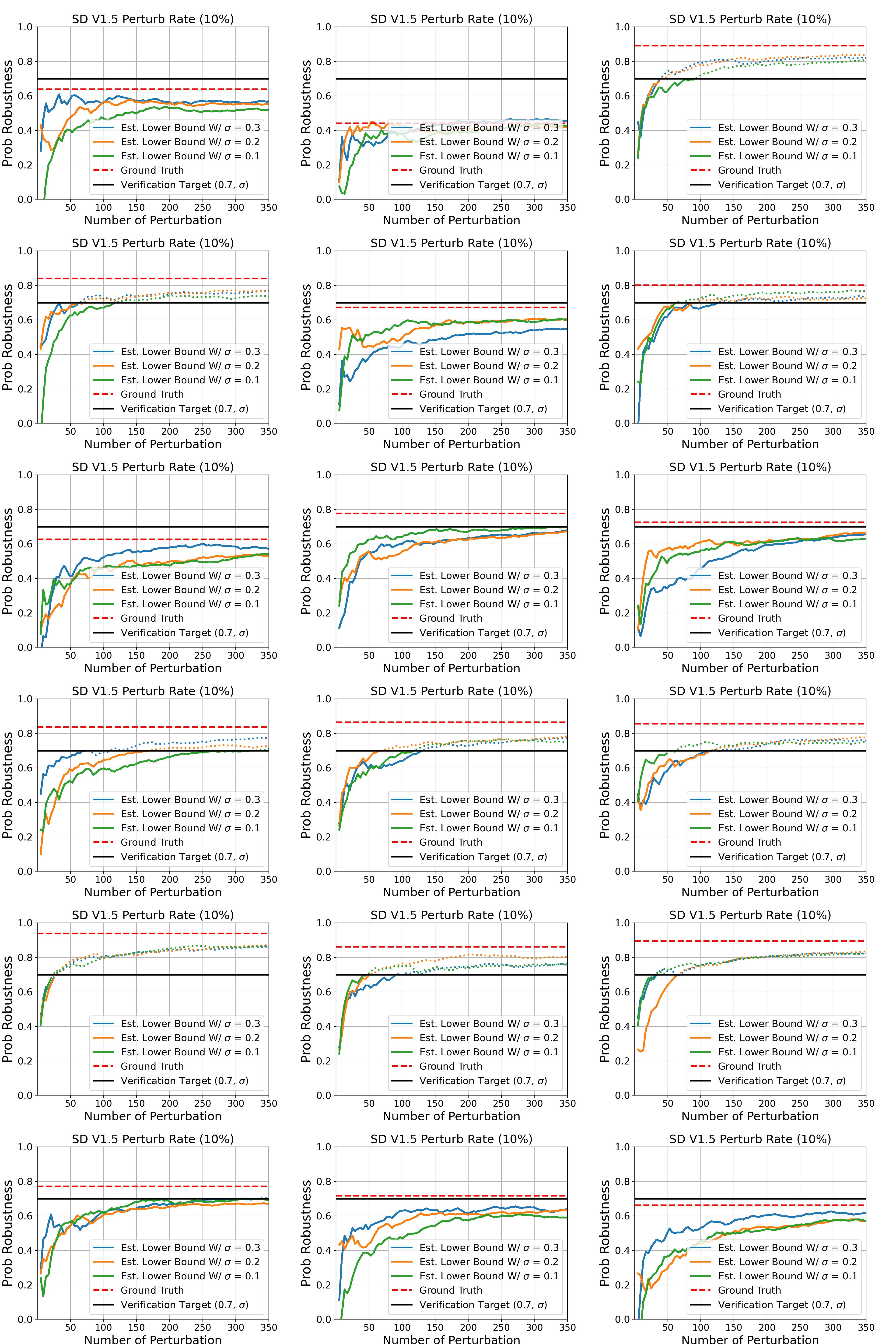}
    \caption{More results on ProTIP effectiveness for SD V1.5 Pert. Rate 10\% (Part 2).}
\end{figure}

\begin{figure}[tb]
    \centering
    \includegraphics[width=\textwidth]{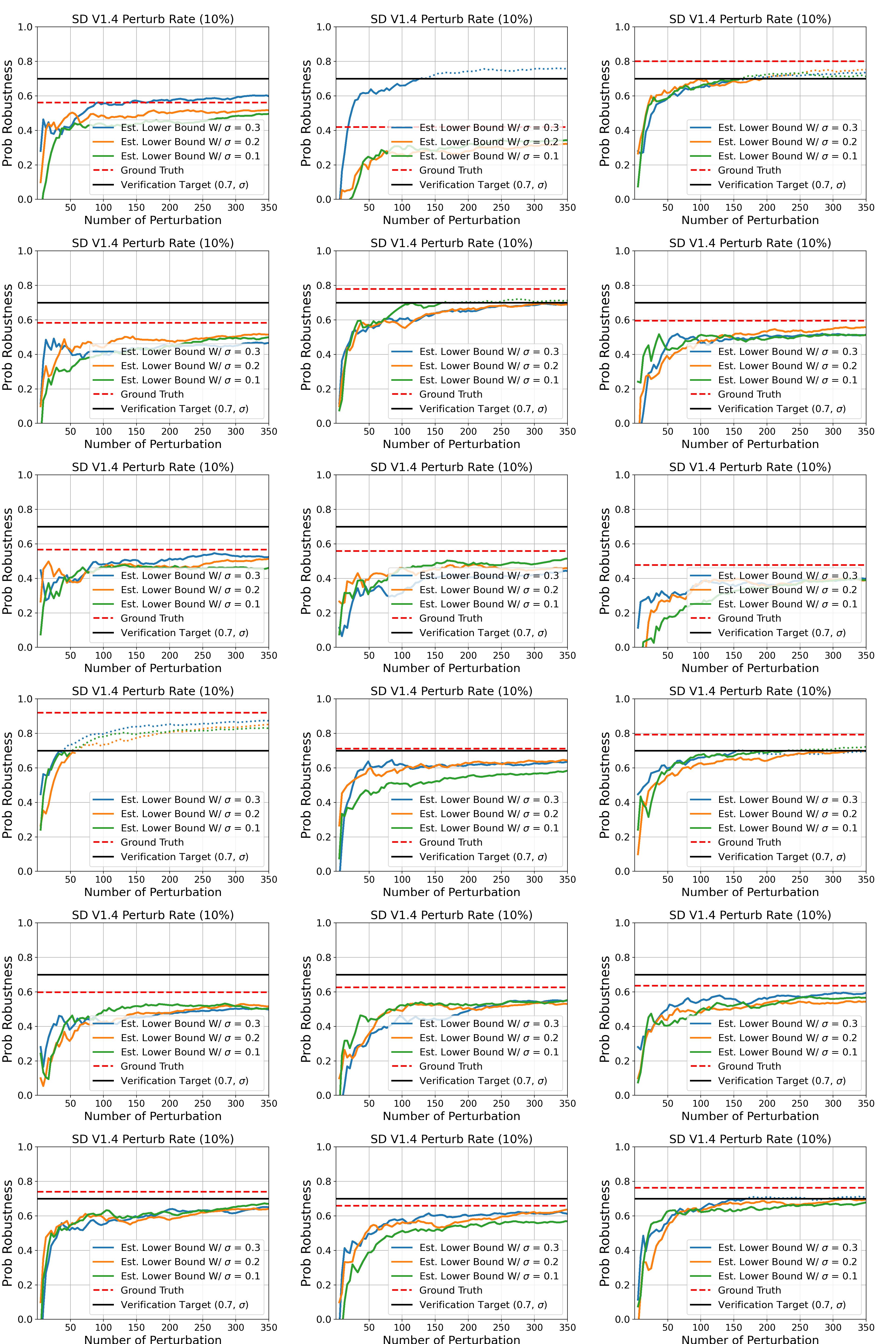}
    \caption{More results on ProTIP effectiveness for SD V1.4 Pert. Rate 10\% (Part 1).}
\end{figure}

\begin{figure}[tb]
    \ContinuedFloat
    \centering
    \includegraphics[width=\textwidth]{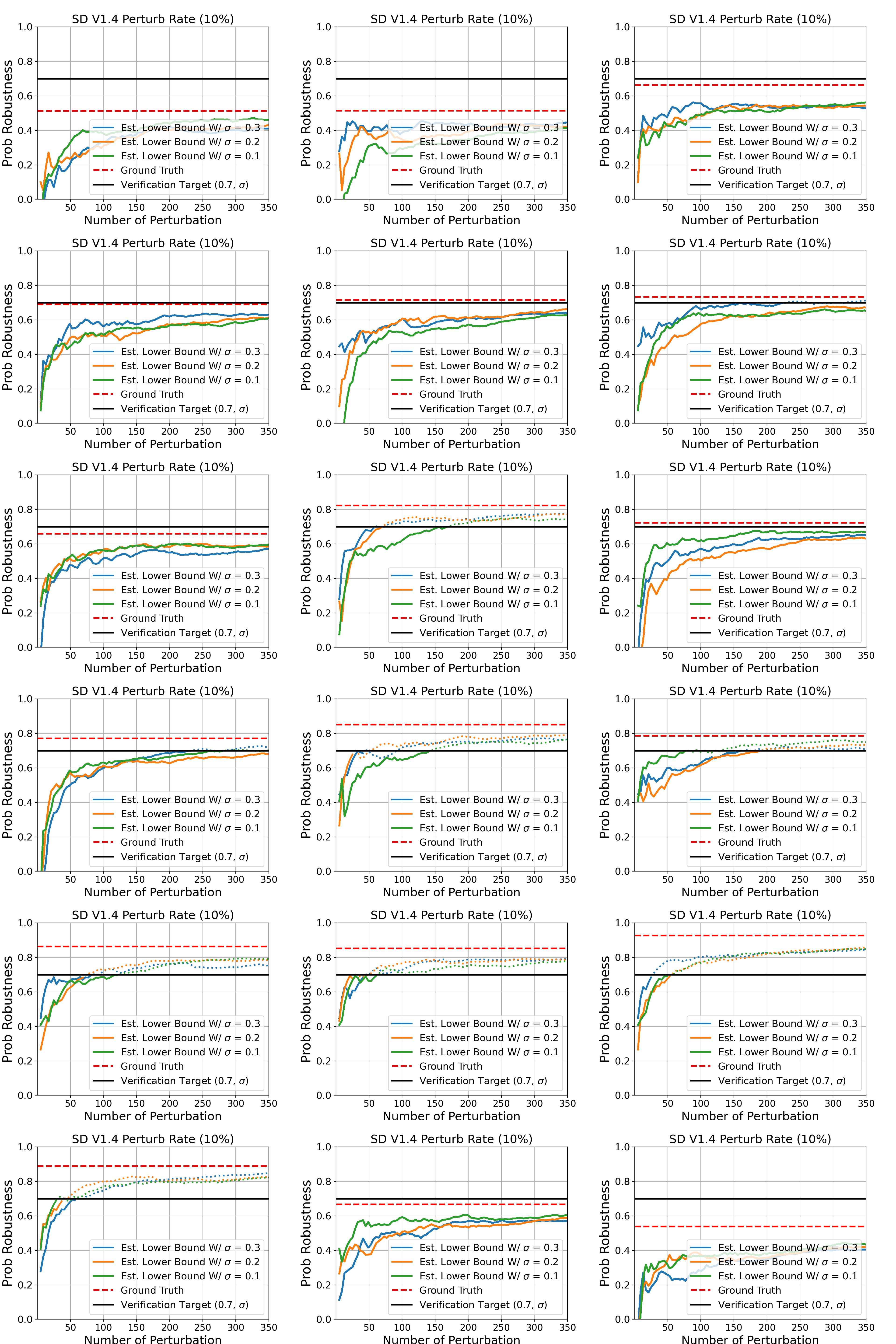}
    \caption{More results on ProTIP effectiveness for SD V1.4 Pert. Rate 10\% (Part 2).}
\end{figure}

\begin{figure}[tb]
    \centering
    \includegraphics[width=\textwidth]{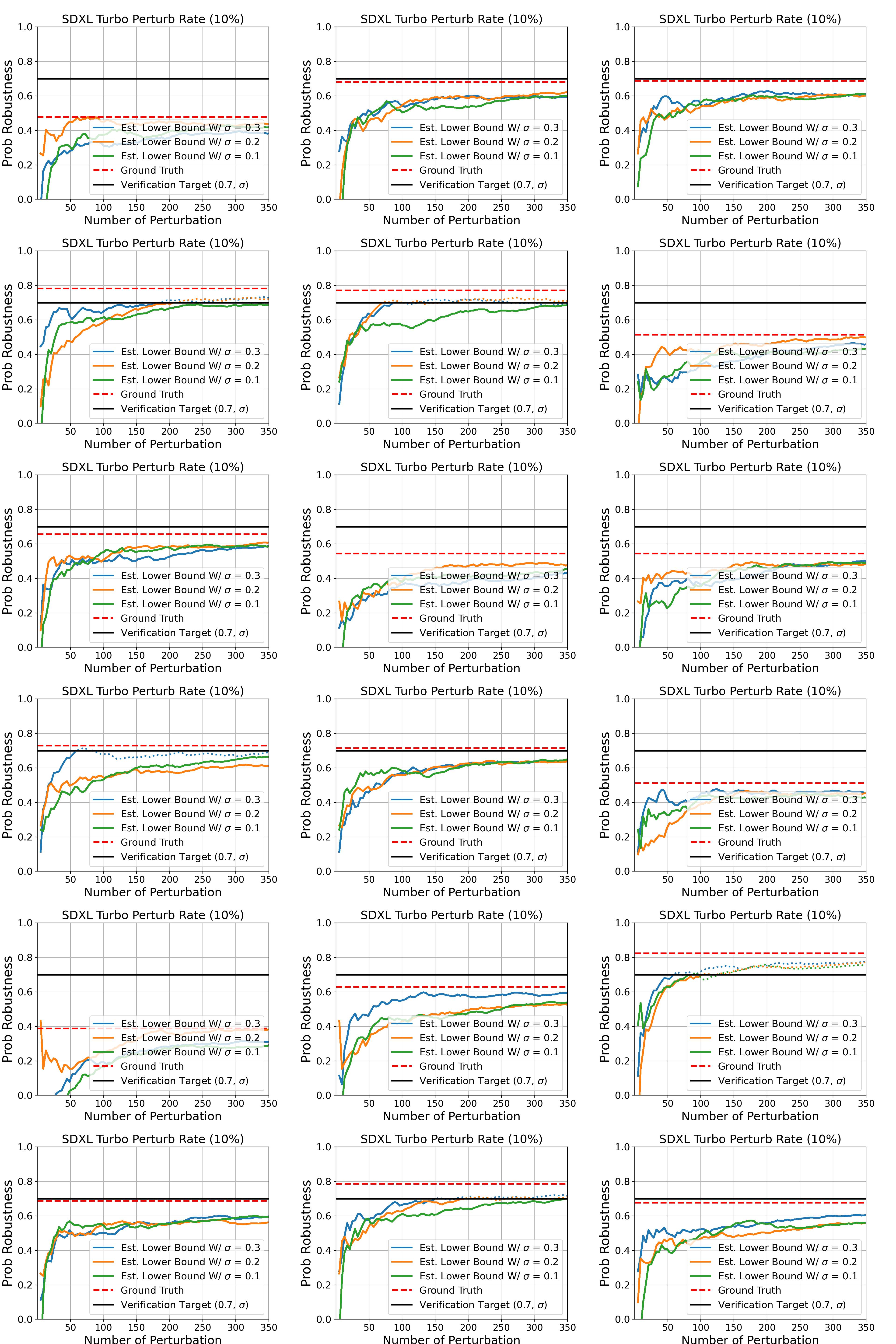}
    \caption{Results on ProTIP effectiveness for SDXL Turbo Pert.~Rate 10\% (Part 1).}
    \label{fig_SDXL}
\end{figure}

\begin{figure}[tb]
    \ContinuedFloat
    \centering
    \includegraphics[width=\textwidth]{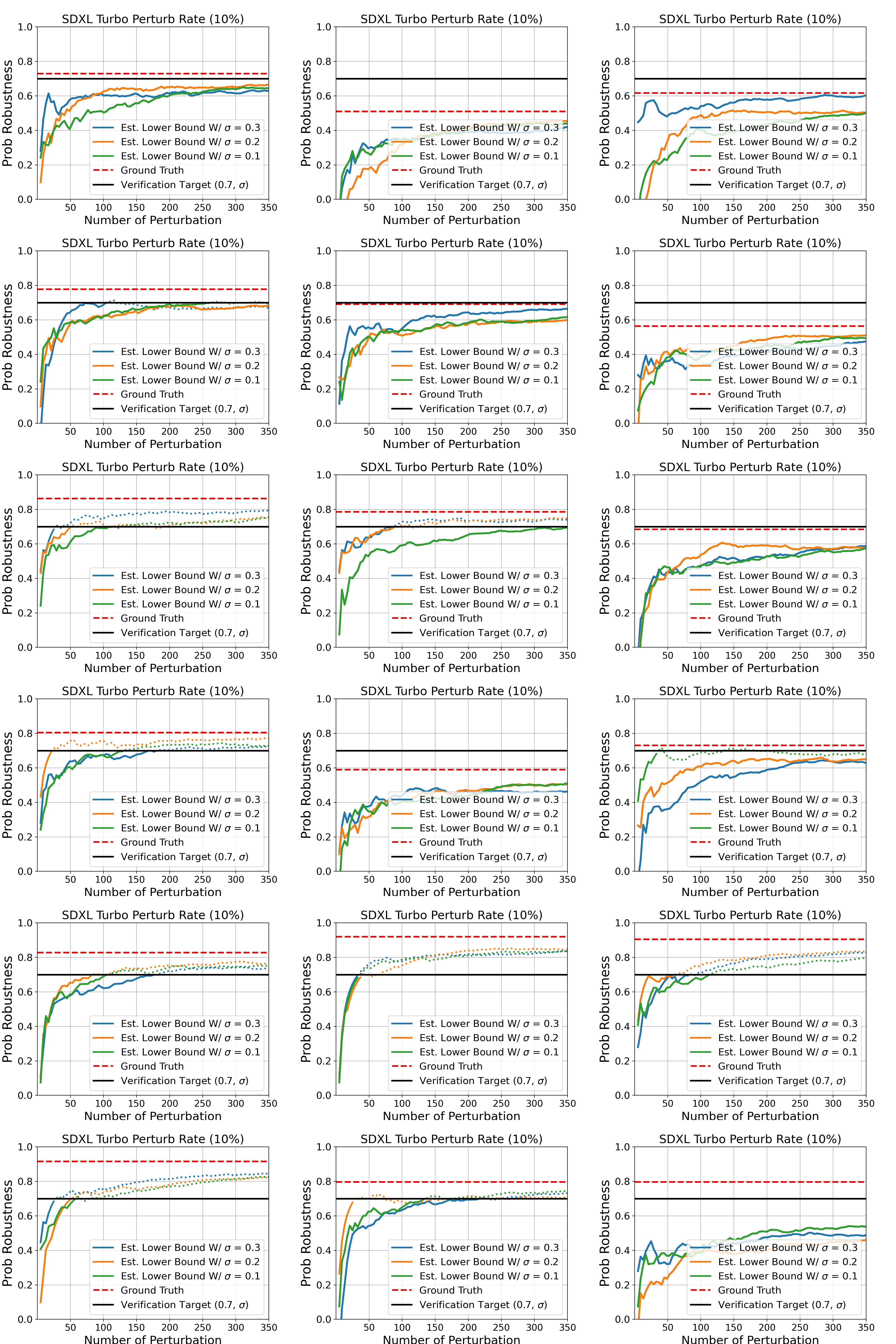}
    \caption{Results on ProTIP effectiveness for SDXL Turbo Pert.~Rate 10\% (Part 2).}
\end{figure}

\begin{figure}[tb]
    \centering
    \includegraphics[width=\textwidth]{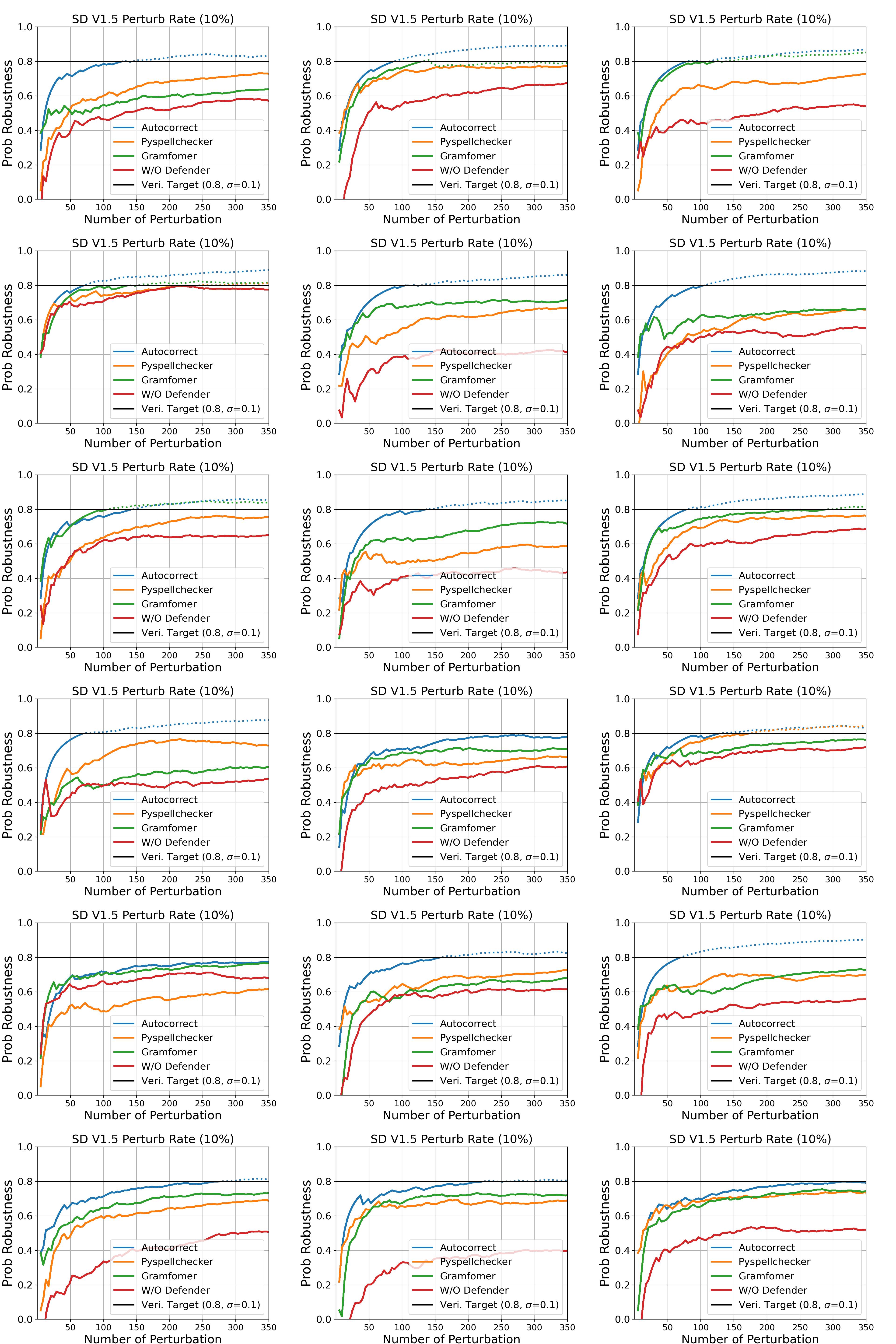}
    \caption{Probabilistic robustness with/without defence methods (Part 1).}
\end{figure}

\begin{figure}[tb]
    \ContinuedFloat
    \centering
    \includegraphics[width=\textwidth]{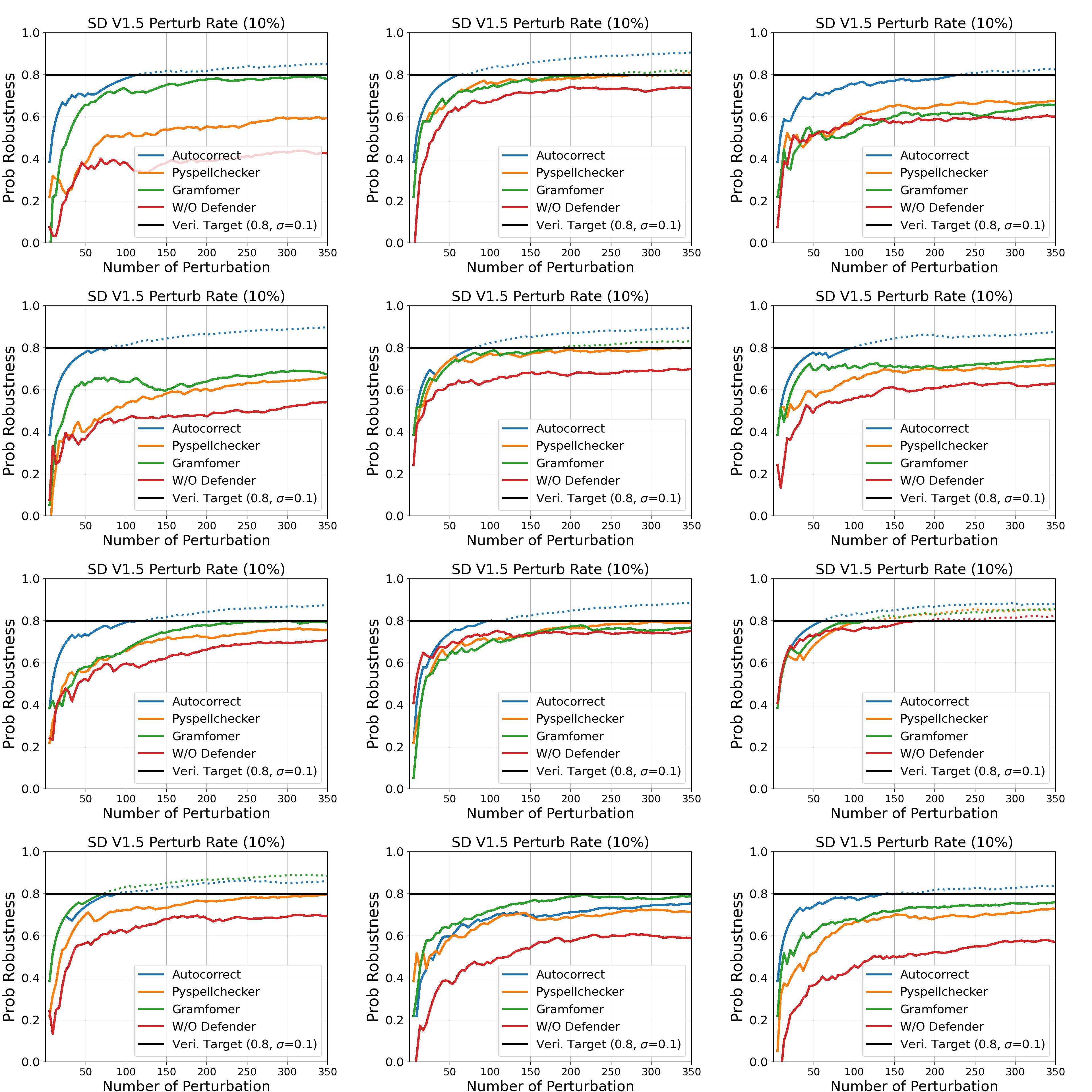}
    \caption{Probabilistic robustness with/without defence methods (Part 2).}
\end{figure}

\end{document}